\documentclass{article}

\usepackage{ijcai26}

\usepackage{times}
\usepackage{soul}
\usepackage{url}
\usepackage[hidelinks]{hyperref}
\usepackage[utf8]{inputenc}
\usepackage[small]{caption}
\usepackage{graphicx}
\usepackage{amsmath}
\usepackage{amsthm}
\usepackage{booktabs}
\usepackage{algorithm}
\usepackage{algorithmic}
\usepackage[switch]{lineno}

\usepackage{amsmath,amssymb,amsfonts}
\usepackage{amsthm}

\newtheorem{theorem}{Theorem}

\theoremstyle{definition}
\newtheorem{definition}{Definition}

\usepackage{graphicx}
\usepackage{textcomp}
\usepackage{xcolor}
\def\BibTeX{{\rm B\kern-.05em{\sc i\kern-.025em b}\kern-.08em
    T\kern-.1667em\lower.7ex\hbox{E}\kern-.125emX}}

\title{Are Graph Attention Networks Able to Model Structural Information?}

\author{
Farshad Noravesh$^1$
\and
Reza Haffari$^2$\and
Layki Soon$^3$\And
Arghya Pal$^4$\\
\affiliations
$^{1,2,3,4}$Monash University\\
\emails
\{Farshad.Noravesh, Gholamreza.Haffari\}@monash.edu,
soon.layki@monash.edu,
arghya.pal@monash.edu
}

\begin{document}

\maketitle

\begin{abstract}
Graph Attention Networks (GATs) have emerged as powerful models for learning expressive representations from such data by adaptively weighting neighboring nodes through attention mechanisms. However, most existing approaches primarily rely on node attributes and direct neighborhood connections, often overlooking rich structural patterns that capture higher-order topological information crucial for many real-world datasets. In this work, we present the Graph Structure Attention Network (GSAT), a novel extension of GAT that jointly integrates attribute-based and structure-based representations for more effective graph learning. GSAT incorporates structural features derived from anonymous random walks (ARWs) and graph kernels to encode local topological information, enabling attention mechanisms to adapt based on the underlying graph structure. This design enhances the model’s ability to discern meaningful relational dependencies within complex data. Comprehensive experiments on standard graph classification and regression benchmarks demonstrate that GSAT achieves consistent improvements over state-of-the-art graph learning methods, highlighting the value of incorporating structural context for representation learning on graphs.
\end{abstract}

\section{Introduction}
In graph neural networks (GNNs), message passing is a fundamental mechanism for aggregating information from neighboring nodes, enabling effective learning on graph-structured data. However, traditional message-passing schemes often suffer from limitations such as oversmoothing and limited receptive fields \cite{KhangNguyen2022}. Combining random walk (RW) techniques with message passing offers a powerful approach to enhance GNN performance by capturing long-range dependencies while preserving local structural information \cite{DexiongChen2024}. RW enables nodes to sample diverse neighborhoods beyond immediate neighbors, facilitating more expressive feature propagation \cite{DiJin2022}. 
\par
Structural embedding is essential in message passing for GNNs because it provides a way to encode the topological properties of nodes within the graph. In standard message-passing frameworks, a node aggregates information from its neighbors to update its representation. However, without structural embeddings, this process may fail to capture higher-order connectivity patterns, leading to suboptimal representations, especially in graphs with complex structures or heterophily \cite{XinZheng2022} . Structural embeddings, such as positional or walk-based embeddings, encode information about a node’s role and position within the graph, enhancing the expressiveness of message passing. This allows GNNs to generalize better across different graph structures and improve performance in tasks like graph classification, node classification and link prediction. The present work is only focused on graph classification and regression. The closest approach to our work for combining RW and message passing is \cite{DexiongChen2024} which aggregates the RW embeddings and sends messages from these aggregated embeddings and finally updates the node representation. We take a different approach and generalize the graph attention to enforce the graph to attend to different structural patterns of neighbour nodes in a data driven way. Moreover, Our approach employs a preprocessing step based on Word2Vec training \cite{Mikolov2013} to obtain embeddings of ARWs.
\par
In graph attention network(GAT) \cite{Velickovic2018}, walk length influences the adaptive weighting of neighbors through the attention mechanism. Unlike GCN, which uniformly aggregates features, GAT assigns different importance to nodes in the neighborhood, mitigating the over-smoothing problem to some extent. Longer walks in GAT allow the model to capture more distant relationships, but attention scores decay as the distance increases, reducing their impact. Shorter walks, on the other hand, focus on local node interactions, leveraging the attention mechanism to refine feature aggregation within a limited neighborhood. While GAT is generally more robust to over-smoothing than GCN, excessive walk lengths can still introduce noise and increase computational complexity. Therefore, finding an optimal walk length remains essential for effectively utilizing GAT in graph learning tasks.
\par
In contrast to ARW, Graph Kernel Neural Networks (GKNNs) are more flexible in that they can incorporate node labels or attributes directly into the kernel computation, enabling them to distinguish between nodes with different semantic meanings. This makes GKNNs suitable for applications where node identities or features carry important information, such as in chemistry, where atom types (node labels) significantly impact the molecular properties. While ARW focuses on invariant structural patterns, GKNNs can blend both structural and attribute-based similarities, offering a more detailed graph representation when node-specific information is relevant. However, this flexibility often comes with higher computational costs compared to the sampling-based simplicity of ARW.
\par
In this paper, we leverage both ARW and GKNNs to capture structural information for molecular classification and graph regression tasks respectively. Note that we first train them and use them as a pretrained model to be used for GSAT. We used QM9, ZINC dataset for graph property prediction. We then use these structural node representation and introduce a novel model GSAT that learns how to attend to different nodes based on its structural representation. This attention allows node features to be passed in a message passing system that is controlled by structural information.

\subsection{Main Contributions}
The contributions of our work are as follows:
\begin{enumerate}
\item{We introduced two methods to produce a pretrained model for structural node embedding needed for downstream tasks.}
\item{We generalized the GAT to the case that combines structural information with node attributes to guide the message passing via appropriate learned edge attention weights.}
\item{We outperformed state-of-the-art (SOTA) baselines for graph classification on some benchmarks with only one layer of GSAT which is an indication that higher layers are not necessary if structural information could be captured adequately using ARW.}
\item{Sensitivity analysis is performed to investigate the effect of ARW hyperparameters such as walk length and the size of structural features on the graph classification performance.}
\end{enumerate}

\section{Related Work}
\subsection{Graph Kernels}
\cite{Cosmo2024} uses graph kernels that compute an inner product on graphs, to extend the standard convolution operator to the graph domain and provides structural masks that are learned during the training process. Similarly, \cite{AosongFeng2022} introduced KerGNNs which utilizes trainable hidden graphs as graph filters and are combined with subgraphs centered at each node to update node embeddings using graph kernels. The drawback of these types of kernels is the limited assumption on the number of learnable structures. Even a big number does not resolve the issue since many of the structures would then have high correlation with each others. Graph kernel methods don’t scale well to large graphs. The drawback of methods like \cite{Cosmo2024} is the lack of modeling for node neighbour structures based on label information since \cite{Cosmo2024} has focused on graph classification tasks only. Many methods such as \cite{Kalofolias2021},\cite{AosongFeng2022},\cite{Cosmo2024} that use graph kernels to model structural similarity of two nodes are ignoring the node labels as a way to model local structure and therefore can not fully capture heterogeneous graphs or tasks like node classification. 
\subsection{Random Walks}
\cite{WilliamHamilton2017} unified many graph representation learning methods such as deepWalk \cite{perozzi2014deepwalk}, Node2Vec \cite{grover2016node2vec}, and GraphSage in a framework that implements encoder, decoder, similarity measures and loss functions distinctly. \cite{YuTian2019} leverages kernels instead of encoder-decoder architecture in  \cite{WilliamHamilton2017} and implements the kernel between two nodes using feature smoothing method of Nadaraya-Watson kernel weighted average. Methods in \cite{YuTian2019} and \cite{WilliamHamilton2017} ignore the local structure of two nodes  and optimizes node embeddings so that nearby nodes in the graph have similar embedding. In many applications, two nodes that are far from each other in the global positioning may have very similar local structures such as having similar number of triangle structures. \cite{Ribeiro2017} resolved this research gap by introducing struc2vec that generates structural context for nodes. The core of struc2vec is a variable that measures the ordered degree sequence of a particular set. The set is the ring of nodes at distance k. Then the structural distance between any two nodes can be obtained recursively by measuring the distance between two ordered degree sequences corresponding to the two nodes. Another type of structure arises in heterogeneous graphs. \cite{XuanGuo2024} proposes heterogeneous anonymous walk (HAW) for representation learning on heterostructures. HAW could be seen as generalization of ARW . Thus, it maps to the same ARW in the original formulation of ARW that can distinguish two different sequences by concatenating them with node types.  
\par
Methods so far do not integrate the rich RW representations with message passing methods. To address this research gap, \cite{DexiongChen2024} put forward a novel framework that integrates them by aggregating RW embeddings and learns the encoding of RW end-to-end. However, they neglect the usage of ARW to make their modeling more generalisable. Another drawback of \cite{DexiongChen2024} is the limitation in walk embedding that the entries in the vector are limited to two sequential node embedding which neglects the richness of the whole sequence representation and cuts off the nonlocal information in the sequence since each sequence embedding can be analogous to sentence embedding in natural language processing(NLP).

\section{Background}

\begin{figure*}[t]
    \centering
    \fbox{\includegraphics[width=0.6\textwidth]{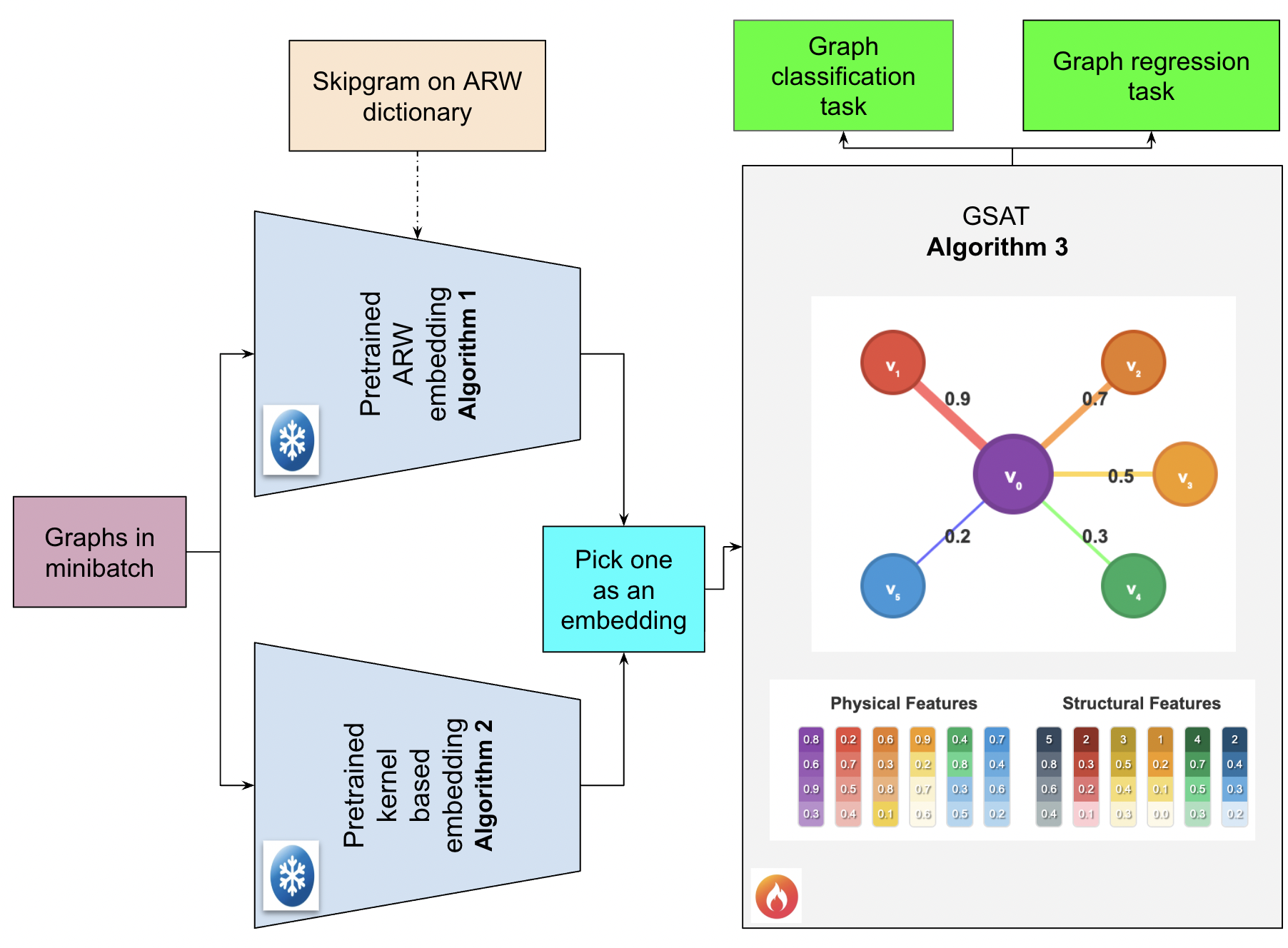}}
    \caption{The proposed framework to model structural information for the downstream tasks.}
    \label{fig:framework}
\end{figure*}

\subsection{Preliminaries}
Given a graph $G = (V, E)$, we use $V$ and $E$ to denote its nodes and edges, respectively. The nodes are indexed by v and u
such that $v, u \in V$, and an edge connecting nodes v and u is denoted by $(v, u) \in E$. The connectivity is encoded in
the adjacency matrix $A \in R^{n\times n}$  where n is the number of nodes. p denotes the width (hidden dimension size), while $l$
is the number of layers. The feature of node v at layer $l$ is written as $h^{l+1}_{v}$. In GSAT, we prefer to have either one or very few layers since the problem of oversmoothing deteriorates the significance of node and graph representations.
\subsection{Anonymous Random Walk} 
ARW was originally introduced in \cite{SilvioMicali}. \cite{Ivanov2018} designed task-independent algorithms for learning graph representations based on ARW. ARW methods inherently require permutation invariance because they are designed to capture the structural essence of a graph without relying on specific node identities. In many domains such as social networks or molecular graphs, the exact identity of a node is arbitrary, and what truly matters is the role it plays within the structure, such as being a hub, bridge, or part of a cycle. ARW methods achieve this by anonymizing the walk, tracking only the sequence of structural positions (e.g., node degrees or relative positions) rather than node labels or IDs. This makes them especially effective in tasks where the graph’s topology and role-based patterns are more informative than the individual node attributes. GSAT will receive these frozen learned structural node representations as inputs. Here we give two definitions that are needed to create ARW.
\begin{definition}[\cite{Ivanov2018}]
Let $s = (u_{1}, u_{2}, \ldots, u_{k})$ be an ordered list of elements with $u_i \in V$. 
We define the positional function $\textbf{pos}(s, u_i) \to q$ such that, for any ordered list $s = (u_{1}, u_{2}, \ldots, u_{k})$ and an element $u_i \in V$, it returns a list $q = (p_{1}, p_{2}, \ldots, p_{l})$ of all positions $p_j \in \mathbb{N}$ at which $u_i$ occurs in the list $s$.
\end{definition}
\begin{definition}\cite{Ivanov2018}
If $w = (v_{1}, v_{2}, \ldots , v_{k})$ is a random walk, then its corresponding anonymous walk is the sequence of integers $a = (f(v_{1}), f(v_{2}),\ldots, f(v_{k}))$, where integer $f(v_{i}) = \textbf{min} \ \textbf{pos}(w, v_{i})$.
\end{definition}
The aim is to maximize the following average log probability:
\begin{equation}
\frac{1}{T} \sum_{t=\delta}^{t=T-\delta} \log p (w_{t}|w_{t-\delta},\ldots,w_{t+\delta},d) 
\end{equation}
where the graph corresponds to d and $\delta$ is the window size, i.e. number of context words for each target word. The above probability is defined via the following softmax function:
\begin{equation}\label{eq:scoring}
p(w_{t}|w_{t-\delta},\ldots,w_{t+\delta},d) = \frac{e^{y(w_{t})}}{\sum_{i=1}^{\eta}e^{y(w_{i})}}
\end{equation}
where $\eta$ is the number of ARWs of length $l$. Equation~\ref{eq:scoring} is a softmax and $y(w_{i})$ is the scoring function for the candidate walk $w_{i}$.Note that the present work simply uses skip-gram to represent each ARW and is limited to homogeneous graphs.
\section{Methodology}
Given a graph and its node attributes, the motivation is to use the structural representation(SR) to learn how to guide the message passing of original attributes(OA) in a data driven way. Note that combinatorially each individual node may be surrounded by many diverse substructures and the model is learning how to provide a gating mechanism to control the message passing based on local structures since encoding different substructures around each node provides better generalization and expressivity of the model. 
We define latent structure representation(LSR) as the representation of each node such that implicit local structures could be represented as a vector and GSAT provides such message passing algorithm. 
\subsection{Preprocessing}
We use two different ways to model structural information namely ARW method and graph kernel method. Section~\ref{sec:RW} describes how to train a skipgram model to obtain structural information around each node. Section~\ref{sec:filters} explains how filters are learned and structural node representations are achieved using this graph kernel method. 
While GKNN might capture specific substructures like rings and chains with labels (atoms), ARW is useful when structure matters more than atom identity when discovering common structural roles or motifs are significant. 

\subsubsection{Anonymous Random Walks}\label{sec:RW}
RWs (such as those used in Node2Vec \cite{AdityaGrover2016} or DeepWalk\cite{Perozzi2014}) can be interpreted similarly to SkipGram in the sense that nodes in the graph are treated as words, and the RW serves as the context that SkipGram attempts to predict. Just as SkipGram learns the relationship between a target word and its context words, graph-based random walk methods learn the relationships between nodes and their neighbors, encoding the local graph structure into meaningful node embeddings. Thus, RWs in graph learning methods serve a role analogous to context windows in SkipGram, both helping to capture local structure for effective representation learning.

\par
Skipgram is used in graph embedding methods to learn meaningful node representations. Skip-gram, originally used in word embeddings (e.g., Word2Vec \cite{TomasMikolov2013}), learns vector representations by maximizing the likelihood of predicting context nodes given a target node. In the context of graph embedding, anonymous random walks generate sequences of node visits that capture structural properties of the graph without considering specific node identities. These sequences serve as input to the skip-gram model, treating sequences of RW similarly to words in a sentence. By optimizing the skip-gram objective on these walks, the model learns embeddings that capture local and global structural relationships in the graph.
We use skipGram which is a fast Word2Vec algorithm \cite{TomasMikolov2013}. The ARW could be seen as a word in a sentence since there are many ARW starting from a fixed node. Preprocessing is done to calculate word embedding through skipGram algorithm. The overall algorithm is shown in Algorithm~\ref{alg:arw}. Just like in NLP, each unique pattern $w \in\mathcal{A}$ is treated like a word in a vocabulary. The mapping $\phi$ is a lookup table (implemented as an embedding matrix), initialized randomly and learned through Skip-gram training. Note that this algorithm is fully unsupervised and no knowledge of graph labels is needed. A simple step by step example is shown in appendix~\ref{ap:simple} to show how walks are mapped and embedded. 

\begin{algorithm}
\small
\caption{Anonymous Walk-Based Structural Node Embedding with Skip-Gram}
\label{alg:arw}
\begin{algorithmic}[1]
\REQUIRE Graph $G = (V, E)$, walk length $l$, number of walks per node $r$, context window size $w$, embedding dimension $d$
\ENSURE Structural node embeddings $Z \in \mathbb{R}^{|V| \times d}$
\STATE Initialize embedding table $\Phi: \mathcal{A} \rightarrow \mathbb{R}^d$ for all anonymous walk patterns $\mathcal{A}$
\STATE Initialize empty multiset of training pairs $\mathcal{T}$
\FOR{each node $v \in V$}
    \STATE Sample $r$ random walks of length $l$ starting from $v$
    \STATE Convert each walk to its anonymous walk representation
    \STATE Store resulting walk sequence as $S_v = [w_1, w_2, \dots, w_r]$
    \FOR{each position $i$ in $S_v$}
        \FOR{each $j \in [i - w, i + w],\ j \ne i$ and $j$ in bounds}
            \STATE Add pair $(w_i, w_j)$ to training set $\mathcal{T}$
        \ENDFOR
    \ENDFOR
\ENDFOR
\STATE \texttt{// Skip-gram training}
\FOR{each training pair $(w_i, w_j) \in \mathcal{T}$}
    \STATE Sample negative walks $\{w_k\}$ from $\mathcal{A}$ uniformly
    \STATE Update $\Phi$ using gradient step on:
    \[
        \log \sigma(\Phi(w_i)^\top \Phi(w_j)) + \sum_{w_k \in \mathcal{N}} \log \sigma(-\Phi(w_i)^\top \Phi(w_k))
    \]
\ENDFOR
\STATE \texttt{// Node-level aggregation}
\FOR{each node $v \in V$}
    \STATE Let $S_v = [w_1, w_2, \dots, w_r]$ be anonymous walks from $v$
    \STATE Compute embedding: \quad $Z_v \gets \frac{1}{r} \sum_{i=1}^r \Phi(w_i)$
\ENDFOR
\RETURN $Z$
\end{algorithmic}
\end{algorithm}
The length of the walk can vary, and multiple walks are often sampled to capture diverse structural patterns within the graph. By sampling a series of RWs, it is possible to explore local neighborhoods of nodes, which can then be used in learning meaningful embeddings or representations of the graph's structure. To obtain LSR of each node, an ARW is drawn randomly.

\begin{algorithm}
\small
\caption{Learning structural masks with parameters $\theta$, and the resulting node structural embedding}
\label{algo:GK}
\begin{algorithmic}[1]
\REQUIRE ${\mathcal{G}_{1},\mathcal{G}_{i},\ldots, \mathcal{G}_{N}}$ , $y_{i}\in \mathbb{R}$, \text{proto size}, $m$
\ENSURE The structural masks $\{M_{1},M_{2},\ldots,M_{m}\}$ are learnable
\FOR{$i = 1$ to $N$}
\FOR{$j = 1$ to $m$}
\STATE $z_{j}^{i}(v) \gets \mathbb{E}_{G \sim M_{j}}K(G, \mathcal{N}^{r}_{\mathcal{G}_{i}}(v))$ \hfill   \textit{for j-th structural mask}
\ENDFOR
\STATE $F^{i}(v) \gets f_\phi(z^{i}(v))$ \hfill \textit{Apply MLP to each node representation}
\STATE $F_{\mathcal{G}_{i}} \gets \frac{1}{|V|} \sum_{v \in V} F^{i}(v)$ \hfill \textit{Mean pooling over nodes}
\ENDFOR
\STATE  $\mathcal{L}_{\text{tot}} = \mathcal{L}_{\text{MSE}} + \mathcal{L}_{\text{JSD}}$  \hfill \textit{total loss}
\RETURN $\theta, \phi, M_{i}, z$ 
\end{algorithmic}
\end{algorithm}

 The mean vector of all word embeddings of an ARW started at node $v$ will be the LSR in of that node in GSAT and we call it $h_{v}^{(s)}$. 
\subsubsection{Learning Filters}\label{sec:filters}. The second method that we use to model structural information efficiently is graph kernel which models the similarity of the two subgraphs around two nodes, and learns the filters similar to \cite{Cosmo2024}.
\cite{Cosmo2024} leverages crossentropy loss to be able to backpropagate the gradient and learn the filters, but we use mean squared error(MSE) loss instead, since graph property values are available for datasets like QM9. The structural embedding is as follows:
Given a graph $G = (V, X, E)$ with $|V|$ nodes, we extract $|V|$ subgraphs $\mathcal{N}^r_G(v)$ of radius $r$, each centered at a node $v \in V$. Each subgraph includes the central node $v$, all nodes within distance $r$, and the edges among them. These subgraphs are compared to a set of learnable structural masks $ \{M_1, \dots, M_m\} $ via a kernel function. \begin{equation}
z_j(v) = \mathcal{K}(M_j, \mathcal{N}^r_G(v))
\end{equation}
where $z(v) \in \mathbb{R}^m_{\geq 0}$ is a non-negative, real-valued vector containing the kernel responses for each mask.
Graph kernels, which operate on discrete graph structures, are generally non-differentiable. To enable optimization of the structural masks, \cite{Cosmo2024} relaxes them into distributions over graphs and computes the expected kernel value under these distributions.
Each structural mask $M_i = (V_{M_i}, X_{M_i}, E_{M_i})$ is a \textit{first-order stochastic graph}, defined as a probability distribution over graphs. Each edge $e \in E_{M_i}$ is sampled independently from a Bernoulli distribution with parameter $\theta_e$ and each node label
$x \in X_{M_i}$ is sampled from a discrete probability distribution over a dictionary $\mathcal{D}$, with parameters 
$\phi^x_d$ denoting the probability that $x$ takes value $d \in \mathcal{D}$. The parameters $\theta$ (for edges) and $\phi$ (for node labels) are the learnable components of the Graph Kernel Convolution (GKC) layer. Note that we used Gumbel-Softmax trick to give a differentiable approximation to sampling from a categorical distribution. After taking expectation from the distribution we would have:
\begin{equation}\label{eq:struct-embedding}
z_{j}(v)=\mathbb{E}_{G \sim M_{j}} \mathcal{K}(G, \mathcal{N}^{r}_{\mathcal{G}}(v))
\end{equation}
$\mathcal{N}^{r}_{\mathcal{G}}(v)$ in equation~\ref{eq:struct-embedding} is the r-hop neighbourhood of node $v$ and $M_{j}$ is the j-th filter that should be learned. Thus, if we have $m$ filters(structural masks) like $\{M_{1},\ldots,M_{m}\}$, then the size of vector $z$ will be $m$. 
The full training algorithm is shown in Algorithm~\ref{algo:GK}. After calculating node structural embedding, we give it to a multilayer perceptron (MLP) and calculate a mean pooling to obtain the final graph representation. The following MSE loss is used for regression:
\begin{equation} \label{eq:MSEloss}
\mathcal{L}_{\text{MSE}} = \mathbb{E}_{\theta} \frac{1}{n} \sum_{i=1}^{N} (y_i - F_{\mathcal{G}_{i}})^2
\end{equation}
Since we can approximate probability distribution induced by the $j$-th mask over the graph nodes as 
$P_{j} = \{\frac{z_{j}(v)}{\sum_{v'}z_{j}(v')}|v,v' \in V \}$. With the same spirit as \cite{LuBai2020},\cite{Englesson2021}, \cite{LucaCosmo2025}, we maximize the Jensen–Shannon divergence (JSD) to force the learned distributions to be as maximally distant from one another to create more unique structural masks:
\begin{equation}\label{eq:JSD}
\mathcal{L}_{JSD} = -H(\sum_{j=1}^{m}P_{j})+\sum_{j=1}^{m}H(P_{j})
\end{equation}
where $H(P)$ in equation~\ref{eq:JSD} is the Shannon entropy.
Thus, the total loss that is needed to learn the parameters of the structural masks is as follows:
\begin{equation}
\mathcal{L}_{\text{tot}} = \mathcal{L}_{\text{MSE}} + \mathcal{L}_{\text{JSD}}
\end{equation}

Figures~\ref{fig:5nodefilters},\ref{fig:6nodefilters} show some of the learned filters for masks having five nodes and six nodes respectively. They are obtained after training using algorithm~\ref{algo:GK} on QM9 dataset. Similarly, some of the learned filters for ZINC dataset are shown in figure~\ref{fig:zincfilters}. These learned filters produce embeddings that can be used at inference time to predict the properties of new unseen graphs. 

\subsection{Graph Structure Attention Network}
In GSAT, structural embeddings are used to inform the attention mechanism to automatically discern which neighbors are more likely to contribute meaningfully to a node’s updated representation based on their structural proximity in the graph or any other implicit guidance of structural embeddings. Note that GSAT uses ARW and not the original RW. Thus, structural encoding is more effective. In a GSAT, these ARW embeddings in algorithm~\ref{alg:arw} or the expected kernel embeddings in algorithm~\ref{algo:GK} are used to create the attention mechanism to automatically figure out which neighbors are more likely to contribute meaningfully to a node’s updated representation based on their structural proximity or any other implicit justification in the graph. This inductive bias guides the message passing of the nodes with original features.
\par
We decouple the feature vector into two different parts namely, structural attributes $h_{u}^{(s)}$ and original attributes $h_{u}^{(orig)}$. ARW embedding in algorithm~\ref{alg:arw} or expected kernel embeddings in algorithm~\ref{algo:GK} are examples of structural attributes that are used in GSAT which have superscript s in our terminology. The present framework allows any other type of structural attributes to be used for $h_{u}^{(s)}$. 
Like other GNNs, deep GATs suffer from over-smoothing, where repeated message passing causes node representations to become indistinguishable, reducing the model’s expressiveness. One cornerstone for the success of GAT is the fact that unlike GCNs, which use fixed-weight averaging, GATs assign different importance (attention scores) to each neighbor, allowing more influential nodes to contribute more to the final representation. Although we draw inspiration from graph attention network(GAT)\cite{Velickovic2018} , the attention weights are not based on original attributes and is only calculated using the $h_{v}^{(s)}$. Thus, the aggregated messages at layer k are:
\begin{equation}
m_{N(u)}^{(k)}=\sigma(\sum_{v \in \mathcal{N}(u)} \alpha_{u,v}Wh_{v}^{(orig)})
\end{equation}
where the attention weights are as follows:
\begin{equation}
\alpha_{u,v} = \frac{\exp (Relu(a^{T}[Wh_{u}^{(s)}||Wh_{v}^{(s)}]))}{\sum_{v'\in \mathcal{N}_{u}} \exp (Relu(a^{T}[Wh_{u}^{(s)}||Wh_{v'}^{(s)}])) }
\end{equation}
Finally, the nodes are updated using the following combine rule:
\begin{equation}
h_{u}^{(k+1)} = ReLU(V^{(k)}m_{\mathcal{N}(u)}^{(k)}+b^{(k)})
\end{equation}
where $V^{(k)}$ denotes a trainable weight matrix and $b^{(k)}$ is bias term. Although the present work only uses $h_{u}^{(s)}$ to calculate the attention weights, any other combination(concatenation or summation) of structural and original nodes is also possible. Attention mechanisms can be noisy or overly focused on certain parts of the input. The outputs of multiple heads are averaged which leads to a more robust representation. Thus, GSAT is implemented based on multiheaded attention with similar spirit to the original multiheaded GAT. The full algorithm is shown in Algorithm~\ref{alg:gsat}. 
\begin{algorithm}
\small
\caption{Training a Graph Structure Attention Network (GSAT)}
\label{alg:gsat}
\begin{algorithmic}[1]
\REQUIRE Graph $G = (V, E)$, feature matrix $X \in \mathbb{R}^{|V| \times d_{\text{in}}}$, labels $Y$, number of heads $H$, learning rate $\eta$, number of epochs $T$
\ENSURE Trained parameters of GSAT
\IF{method is \textbf{ARW} embedding}
    \STATE store node structural representation of Algorithm~\ref{alg:arw} in $h^{(s)}_{v}$
\ELSIF{method is \textbf{GKNN} embedding}
    \STATE store node structural representation of Algorithm~\ref{algo:GK} in $h^{(s)}_{v}$
\ENDIF
\STATE Initialize attention weights $\{W^{(h)}, a^{(h)}\}_{h=1}^H$ for each head
\FOR{$t = 1$ to $T$}
   
\FOR{$l = 1$ to $L$}
    \STATE $H^{(l)} \leftarrow \text{GSATLayer}(H^{(l-1)}, A)$
\ENDFOR

\STATE $H \leftarrow H^{(L)}$ \COMMENT{Final node embeddings}

\STATE $z \leftarrow \text{GraphPooling}(H)$ \COMMENT{Global pooling: mean / sum / attention}

\IF{$\texttt{task} = \texttt{regression}$}
    \STATE $\hat{y} \leftarrow \text{MLP}_{\text{reg}}(z)$
    \STATE $\mathcal{L} \leftarrow \text{MSELoss}(\hat{y}, y)$ \COMMENT{Mean Squared Error Loss}
\ELSIF{$\texttt{task} = \texttt{classification}$}
    \STATE $\hat{y} \leftarrow \text{softmax}(\text{MLP}_{\text{cls}}(z))$
    \STATE $\mathcal{L} \leftarrow \text{CrossEntropyLoss}(\hat{y}, y)$
\ENDIF

\STATE Backpropagate gradients of $\mathcal{L}$ w.r.t. $\{W^{(h)}, a^{(h)}\}$
    \STATE Update parameters using gradient descent:
    \[
    W^{(h)} \leftarrow W^{(h)} - \eta \cdot \nabla_{W^{(h)}} \mathcal{L}, \quad
    a^{(h)} \leftarrow a^{(h)} - \eta \cdot \nabla_{a^{(h)}} \mathcal{L}
    \]
    \STATE $H^{(0)} \leftarrow X$
\ENDFOR
\RETURN $\{W^{(h)}, a^{(h)}\}_{h=1}^H$
\end{algorithmic}
\end{algorithm}
SeMole is still one of the best baselines for property prediction of QM9 dataset and we compared GSAT with it as is shown in table~\ref{tab:qm9_results}. 
\begin{table*}[h]
\centering
\caption{Mean Absolute Error (MAE) for the molecular property prediction benchmark in the QM9 dataset. Units are shown below each property.}
\label{tab:qm9_results}
\renewcommand{\arraystretch}{1.1}
\setlength{\tabcolsep}{5pt}
\begin{tabular}{lcccccccccc}
\hline
\textbf{Model} 
& \begin{tabular}{@{}c@{}}$\mu$\\(D)\end{tabular}
& \begin{tabular}{@{}c@{}}$\alpha$\\(Bohr$^3$)\end{tabular}
& \begin{tabular}{@{}c@{}}$\epsilon_{\text{HOMO}}$\\(meV)\end{tabular}
& \begin{tabular}{@{}c@{}}$\epsilon_{\text{LUMO}}$\\(meV)\end{tabular}
& \begin{tabular}{@{}c@{}}$\Delta\epsilon$\\(meV)\end{tabular}
& \begin{tabular}{@{}c@{}}$\langle R^2 \rangle$\\(Bohr$^2$)\end{tabular}
& \begin{tabular}{@{}c@{}}ZPVE\\(meV)\end{tabular}
& \begin{tabular}{@{}c@{}}$U_0$\\(meV)\end{tabular}
& \begin{tabular}{@{}c@{}}$U$\\(meV)\end{tabular}
& \begin{tabular}{@{}c@{}}$C_v$\\(cal/mol$\cdot$K)\end{tabular} \\
\hline
\textbf{SeMole}$_{\textbf{Pretrained}}$     & 0.204 & 0.317 & 15 & 12 & 34 & 0.752 & 3.23 & 17 & 19 & 0.163 \\
\textbf{GSAT} & \textbf{0.162} & \textbf{0.192} & \textbf{11} & \textbf{9} & \textbf{29} & \textbf{0.461} & \textbf{2.14} & \textbf{14} & \textbf{12} & \textbf{0.109} \\
\hline
\end{tabular}
\end{table*}

\subsection{Computational Complexity}
To calculate the computational complexity of GSAT we break it into two parts namely attention calculation for message passing and the hierarchical pooling based on edgePool \cite{Diehl2019} which is . Assume the structural size has dimension F and H be the number of attention heads , E be the number of edges and N is the number of nodes. Then GSAT has computational complexity of $O(HEF)+O(NlogN)$. 
\begin{table*}[t]
\centering
\caption{Graph classification accuracies (mean ± std, \%). Results are averaged over 10 runs. \textbf{Bold} indicates the best, \underline{underline} indicates the second best.}
\label{tab:graph_classification_results}
\small
\setlength{\tabcolsep}{3pt}
\renewcommand{\arraystretch}{1.05}
\begin{tabular}{@{}lcccc@{}}
\hline
\textbf{Model} & \textbf{MUTAG} & \textbf{Prot.} & \textbf{DD} & \textbf{NCI1} \\
\hline
TopKPool \cite{HongyangGao2019} & 67.61±3.36 & 70.48±1.01 & 73.63±0.55 & 67.02±2.25 \\
ASAP \cite{EashanRanjan2020}    & 77.83±1.49 & 73.92±0.63 & 76.58±1.04 & 71.48±0.42 \\
SAGPool \cite{JunhyunLee2019}   & 73.67±4.28 & 71.56±1.49 & 74.72±0.82 & 67.45±1.11 \\
DiffPool \cite{RexYing2018}     & \underline{79.22±1.02} & 73.03±1.00 & \underline{77.56±0.41} & 62.32±1.90 \\
MinCutPool \cite{Bianchi2020}   & 79.17±1.64 & \textbf{74.72±0.48} & \textbf{78.22±0.54} & \underline{74.25±0.86} \\
GSAT-hp (ours)                 & \textbf{86.33±0.55} & \underline{74.29±0.76} & 77.35±1.52 & \textbf{75.12±1.17} \\
\hline
\end{tabular}
\end{table*}

\begin{table}[h]
\centering
\renewcommand{\arraystretch}{1.2} 
\caption{Mean Absolute Error(MAE) for the molecular property prediction in the ZINC dataset.}
\label{tab:ZINC_comparison}
\begin{tabular}{c c c c}
\hline
\textbf{Model} & \textbf{Penalized logP}  & \textbf{QED} & \textbf{MolWt}  \\
\hline
\textbf{SeMole}$_{\textbf{Pretrained}}$ & 1.34 &  0.162 & \textbf{1.78} \\
\textbf{GSAT} & \textbf{1.29} & \textbf{0.083 } & 2.19 \\
\hline
\end{tabular}
\end{table}

\begin{table}[t]
\centering
\caption{Effect of hyperparameters on Protein graph classification accuracy (mean ± std, \%). Results are averaged over 10 runs. \textbf{Bold} indicates best, \underline{underline} indicates second best.}
\label{tab:sensitivity_results}
\footnotesize 
\setlength{\tabcolsep}{3pt} 
\renewcommand{\arraystretch}{0.8} 
\begin{tabular}{c c c c c}
\hline
\textbf{Cfg} & \texttt{struct\_size} & \texttt{walk\_len} & \texttt{numRW/node} & \textbf{Acc. (\%)} \\
\hline
1 & 10   & 10  & 30  & 71.41 ± 0.91 \\
2 & 10   & 20  & 30  & 71.15 ± 0.69 \\
3 & 50   & 10  & 30  & \underline{74.29 ± 0.57} \\
4 & 50   & 20  & 30  & \textbf{74.92 ± 0.59} \\
5 & 100  & 20  & 30  & 73.74 ± 0.36 \\
6 & 100  & 20  & 60  & 73.51 ± 0.84 \\
7 & 200  & 20  & 60  & 71.27 ± 1.13 \\
8 & 400  & 20  & 60  & 70.83 ± 0.74 \\
\hline
\end{tabular}
\end{table}

\section{Experiments}
Note that in all experiments in the present work, we do not concatenate the structural features with original features. However, the concatenation in equation \eqref{eq-concat} could be considered as a more general formulation which provides a framework for future works. There are two main tasks in the present work. The first task is graph classification on MUTAG, PROTEINS, DD and NCI1 and we used ARW for structural embedding of nodes. The second one is graph property prediction which is a regression task on QM9 and ZINC dataset, and the kernel based embedding is used for structural representation of each nodes.
The hyperparameters used in our experiments is optimized by the values in Table~\ref{tab:hyperparameters}. The learning rate is started at $10^{-3}$ but is gradually reduced by 90 percent every 20 epochs. Five heads are used since a single headed attention produced very noisy results with high variance for the performance. There are other hyperparameters that are mentioned in Table~\ref{tab:sensitivity_results}. As Table~\ref{tab:comparativeStudy_results} shows, the hierarchical pooling version (GSAT-hp) produced better results than the global pooling version(GSAT-gp) as expected since the mean pooling simply eliminate the information provided by the graph topology which is essential for efficient graph classification.Table~\ref{tab:graph_classification_results} shows how GSAT slightly outperforms performance for NCI1 dataset. 
\subsection{Dataset}
MUTAG, PROTEINS, DD, and NCI1 are widely used datasets for graph classification, particularly in bioinformatics and cheminformatics \cite{Yanardag2015}. MUTAG consists of molecular graphs where nodes represent atoms and edges represent chemical bonds, with the task of classifying compounds based on their mutagenic properties. PROTEINS contains protein structure graphs, where nodes correspond to secondary structure elements, and edges represent interactions, aiming to classify proteins into functional categories. DD (Drosophila Development) is a larger and more complex protein dataset, making it useful for evaluating models on diverse biological structures. NCI1, derived from the National Cancer Institute, consists of molecular graphs used to predict anti-cancer activity. Their statistics are shown in Table~\ref{tab:dataset_statistics}.
\subsection{Ablation Study for walk length}
Choosing a walk length for each graph dataset distribution is crucial, particularly in protein graph datasets, where capturing motifs and high-order structures significantly impacts model performance. A carefully chosen walk length helps in effectively capturing these motifs and short walks may primarily encode local residue interactions, while longer walks can reveal higher-order structural patterns. If the walk length is too short, the model may fail to recognize essential long-range dependencies critical for functional characterization. Conversely, excessively long walks may introduce noise by aggregating distant, functionally irrelevant nodes, diluting meaningful structural signals. Therefore, designing the walk length in alignment with the inherent structural properties of the dataset, ensures that graph learning models can accurately capture biologically relevant patterns, while minimizing unnecessary information propagation.
\par
The walk length in random walks influences graph classification in multiple ways, particularly for the PROTEIN and MUTAG datasets as are shown in Table~\ref{tab:ablation_walklength_combined}. Here we explain how different walk lengths impact classification performance. The small walk length(short walks) captures local neighborhood structure which preserves fine-grained structural details and is useful for distinguishing proteins based on small functional motifs but it fails to capture global graph connectivity and global topology. The drawback of large walk length is that it may introduce noise if RW drifts too far from meaningful substructures. On the other hand, it captures the overall graph connectivity and large scale properties but it dilutes the importance of local motifs which are critical for graph classification.
\subsection{Sensitivity Analysis}
The effect of ARW is rooted in three main hyperparameters. The first one is \texttt{structural\_size}, which is the size of structural features that \texttt{skipGram} has been trained on. The second parameter is \texttt{walk\_length}, which is the number of walks starting from each node. Finally, \texttt{num\_RW\_per\_node} is the number of RW that has been done. Note that this parameter is directly related to the corpus size when training the \texttt{skipGram} model. 
Here we study the sensitivity of these parameters on the final graph classification performance. From a qualitative point of view, when \texttt{structural\_size} increases, more structural information around each node is represented and therefore we expect that the performance would be increased. However this increase in performance is limited by computational resource limitations since the attention weights should be calculated from these high dimensional features for all nodes. Similarly, increasing \texttt{walk\_length} can capture local neighbourhood at higher radius and may include bottlenecks in the graphs that are responsible for oversquashing. Increasing \texttt{num\_RW\_per\_node} may reduce the noise of structural modeling and produces a robust representation of structure since nodes with high centrality will be implicitly captured by increasing this hyperparameter. The learned filters in QM9 dataset are shown in Appendix~\ref{ap:QM9dataset}. To evaulate GSAT for ZINC dataset, we compared it with SetMole which is a baseline introduced in \cite{AtiaHamidizadeh2023}. Our comparison is shown in table~\ref{tab:ZINC_comparison}
\section{Conclusion}
We have introduced a novel method called GSAT which could be seen as a generalization of GAT. GSAT leverages the structural embeddings of nodes to guide the attention in message passing to learn to automatically manage the edge strengths in the message passings that are guided by structural information of individual nodes. In the present work we first created pretrained models for structural embedding based on ARW and kernel embedding. These frozen pretrained models are then fed to GSAT as structural features and the model learns how to attend to individual neighbours of nodes in the process of message passing of original features. The experiments are done on graph classification and graph property prediction of some benchmarks.
\clearpage
\bibliographystyle{named}
\bibliography{ijcai26}
\appendix
\section{Physical and structural features in GSAT}
Table~\ref{table:nodefeaturesQM9} shows example of physical features that could be used in GSAT for message passing. Note that the weights for the contribution of each neighbor node to the central node is obtained by computing attention over corresponding structural features. Similarly, the physical node features of ZINC dataset is shown in in Table~\ref{table:zincfeatures} and is different from structural features that are obtained by Algorithm~\ref{alg:arw},\ref{algo:GK}. The attention computation over structural features are obtained in Algorithm~\ref{alg:gsat}. Adding motif counts as features is easy but these handcrafted features are not adequate to model complex structures in a graph. The network is no longer learning motif detection from scratch and it is handled in a precomputed form, which defeats the purpose if the goal is for the model to learn structure autonomously and efficiently. 
\par
Spectral graph theory is another paradigm to model structural features. Eigenvectors are tied to graph size and structure while models trained on one size may not transfer directly to another. The common approach to model structures is adding Laplacian eigenvectors or shortest-path encodings to node input features, but these methods still are very global and can not capture local and semi-global substructures inside a graph. Computing Laplacian eigenvectors is $O(n^3)$  in the worst case (though sparse methods are faster) and shortest-path distances are $O(n^2)$  for unweighted graphs. Laplacian or shortest-path information might still miss certain structural differences detectable only by higher-order methods. 
\section{Kernel Computation}
Kernel computation can be costly $O(N^{2})$ when the filter size is big, since $N$ is the number of rows in the graph filters. Small $N$ captures local substructures while a big $N$ can model semi-global and global substructures. Sampling-based ARWs are relatively efficient.
\par
The Weisfeiler-Lehman (WL) kernel~\cite{shervashidze2011weisfeiler} is a widely used graph kernel that refines node labels by iteratively aggregating information from their neighborhoods. Starting from the original node labels, this process incorporates both categorical features and structural patterns, making it particularly effective for labeled graphs in QM9 and ZINC datasets where node identities (e.g., atom types in molecules) are semantically meaningful. By combining label information with multi-hop structural context, the WL kernel provides strong discriminative power and allows us to compare the similarity between learnable filters and the subgraphs around each node. The label information of nodes that we used is atom types and other features in QM9 and ZINC datasets.
Let $\mathcal{G}$ be a set of graphs with initial node labels  $l^{(0)} : V \to \Sigma$.  
At each iteration $h \geq 1$, node labels are updated by aggregating 
neighborhood labels:
\begin{equation}
l^{(h)}(v) = \text{hash}\!\left( l^{(h-1)}(v), \{\, l^{(h-1)}(u) \mid u \in \mathcal{N}(v) \,\} \right)
\end{equation}
where $\mathcal{N}(v)$ denotes the neighbors of node $v$.  
The Weisfeiler--Lehman kernel between two graphs $G$ and $H$ is then defined as
\begin{equation}
k_{\text{WL}}(G,H) = \sum_{h=0}^{H} k\big( \phi^{(h)}(G), \phi^{(h)}(H) \big)
\end{equation}
where $\phi^{(h)}(G)$ is the feature vector counting node labels 
after $h$ iterations, and $k(\cdot,\cdot)$ is a base kernel (e.g., linear kernel on histograms).

\section{Dataset}
Table~\ref{table:nodefeaturesQM9} summarizes the per-atom node features used to represent molecular graphs in the QM9 dataset. Each atom is encoded as a fixed-dimensional feature vector that captures both its elemental identity and key chemical properties derived from RDKit, such as valence, hydrogen count, and aromaticity. These features provide a chemically meaningful and standardized input representation for graph-based generative and predictive models operating on QM9 molecules.
Note that in our experiments we only used one hot atom types.
\begin{table}[h]
\centering
\caption{Node features in QM9 dataset (per atom).}
\label{table:nodefeaturesQM9}
\begin{tabular}{p{0.12\linewidth} p{0.78\linewidth}}
\hline
\textbf{Idx} & \textbf{Description} \\ \hline
0--4 & One-hot atom type: C, N, O, F, H \\
5    & Atomic number (e.g., 6 for C) \\
6    & Number of hydrogens attached \\
7    & Implicit valence (excluding hydrogens) \\
8    & Aromaticity flag (1 if aromatic) \\
9--10 & Additional RDKit-derived flags (e.g., formal charge, chirality, hybridization) \\
\hline
\end{tabular}
\end{table}

\begin{table*}[h]
\centering
\caption{Statistics of bioinformatics graph classification datasets.}
\label{tab:dataset_statistics}
\small
\renewcommand{\arraystretch}{1.2}
\scalebox{1.0}{
\begin{tabular}{c c c c c}
\hline
\textbf{Dataset}          & \textbf{MUTAG} & \textbf{Proteins} & \textbf{DD}   & \textbf{NCI1}  \\
\hline
\# Graphs                 & 188            & 1,113             & 1,178         & 4,110          \\
\# Classes                & 2              & 2                 & 2             & 2              \\
Avg. \# Nodes per Graph  & 17.9           & 39.1              & 284.3         & 29.8           \\
\hline
\end{tabular}
}
\end{table*}

\begin{table}[h]
\centering
\scriptsize
\caption{Statistics of QM9 and ZINC datasets}
\label{tab:dataset_statistics}
\renewcommand{\arraystretch}{1.1}
\begin{tabular}{c c c}
\hline
\textbf{Property} & \textbf{QM9} & \textbf{ZINC} \\
\hline
\# Molecules          & 133,885      & 250,000+ \\
Atom types            & 5 (H, C, N, O, F) & 9+ (H, C, N, O, S, Cl, etc.) \\
Max atoms             & 29           & $\sim$38 \\
Avg. atoms            & $\sim$18     & $\sim$24 \\
Targets               & 12+ physical props. & QED, logP, SA, etc. \\
\hline
\end{tabular}
\end{table}

\section{Comparison With Baselines}
For fair comparison and reasoning, we developed two version of GSAT namely the GSAT-hp and GSAT-gp which correspond to hierarchical and global pooling respectively.  Note that edgePool \cite{Diehl2019} is used to model hierarchical graph pooling in GSAT-hp. There is one advantage of using edgePool which is the fact that there is no requirement to set the number of clusters in advance and this allows the dataset to naturally find appropriate number of clusters in each pooling layer and respects the distribution of dataset. 
It also shows that performance for MUTAG dataset could be enhanced by 6 percent in comparison with MinCutPool which is a well recognized approach to hierarchical graph pooling \cite{Bianchi2020}. Some other important hierarchical pooling methods are  \cite{HongyangGao2019},\cite{EashanRanjan2020}, \cite{JunhyunLee2019}, \cite{RexYing2018}. 
\par
For the regression task on QM9 and ZINC datasets, we use SeMole which is a semi-supervised approach to junction tree variational autoencoder in \cite{Hamidizadeh2023} and is based on \cite{WengongJin2019}. SeMole sees a molecule as a sum of junction tree and  fine details. These substructures form a junction tree, where nodes are substructures and edges represent how they connect. It first generates the junction tree which is a coarse structure to ensure the overall scaffold is valid. Then it assembles the full molecular graph by adding atom-level details and specific bonds between the substructures. In our comparison, property prediction on QM9 and ZINC is done by learning a property regressor in the latent space. First a molecule is passed through both the tree encoder $z_{T}$ and graph encoder $z_{G}$. Then concatenated into a joint latent representation as follows:
\begin{equation}
z=z_{T}||z_{G}
\end{equation}
A small neural network $f_{\theta}$ which is an MLP is trained to map z into a predicted properties. Finally, variational autoencoder (VAE) reconstruction loss and property prediction loss are combined as follows:
\begin{equation}
\mathcal{L}_{\text{SeMole}} = \mathcal{L}_{VAE}+\lambda ||y-\hat{y}||^{2}
\end{equation}
Table~\ref{tab:comparativeStudy_results} illustrates comparative study of four models on four datasets. Global pooling (gp) is used in all cases for a fair comparison.
\begin{table*}[h]
\centering
\caption{Comparative study of four models on four datasets (mean ± std, \%). Results are averaged over 10 runs. \textbf{Bold} indicates the best, \underline{underline} indicates the second best.}
\label{tab:comparativeStudy_results}
\small
\renewcommand{\arraystretch}{1.2}
\scalebox{0.95}{
\begin{tabular}{c c c c c}
\hline
\textbf{Model} & \textbf{MUTAG} & \textbf{Proteins} & \textbf{DD} & \textbf{NCI1} \\
\hline
GCN-gp         & 80.61 ± 3.36        & 72.48 ± 1.01        & 73.63 ± 0.55       & 67.02 ± 2.25 \\
GAT-gp         & \underline{81.83 ± 1.49} & \underline{73.13 ± 0.63} & \underline{76.58 ± 1.04} & \underline{71.48 ± 0.42} \\
GIN-gp         & 81.67 ± 4.28        & 72.56 ± 1.49        & 71.72 ± 0.82       & 67.45 ± 1.11 \\
GSAT-gp (ours) & \textbf{82.29 ± 2.72} & \textbf{73.92 ± 0.41} & \textbf{76.92 ± 0.39} & \textbf{72.21 ± 0.43} \\
\hline
\end{tabular}
}
\end{table*}

\section{Theoretical Analysis}
The most relevant theoretical analysis to our work is SAGNN in \cite{DingyiZeng2023}. SAGNN creates cut subgraphs, which are derived from the original graph by selectively and iteratively removing edges. These subgraphs spotlight local structural variations that might be lost in full-graph aggregation. It computes the RW return probabilities which is the likelihood that a random walker starting at a root node in a cut subgraph returns to it within a certain number of steps. This effectively encodes rich structural context localized to that node and its substructure. Computing Edge Betweenness Centrality (EBC) for all edges in the original graph is expensive. Doing it iteratively means recomputing EBC after each edge removal, which multiplies the cost $O(kVE)$ where k is the number of removed edges. This step of \cite{DingyiZeng2023} is global graph processing, not localized like many GNN-friendly preprocessing steps, which makes it memory and time intensive for massive graphs. Another problem is that EBC can change drastically if a single high-centrality edge or node is altered. This makes the partitioning unstable since small noise in the graph structure can cause different blocks to be found.
\par
To address these issues, our method does not use EBC, and structural information is captured by either ARW or kernel embedding. Theorem~\ref{th:equivalence} shows that ARW and learning filters in kernel method are equivalent and one can pick one of them and the results and analysis are almost the same.
\begin{theorem}\label{th:equivalence}
Let $G$ be the space of subgraphs which includes the subgraphs around nodes of a graph and let $A_{l}$ denote the finite set of possible AWs of fixed length $l$.  Then there exists at least one anonymous walk kernel and the structural node embedding defined by it is a special case of GKNN node embedding.
\end{theorem}
\begin{proof}
For any subgraph $G_{i}$ around node $i$ define the feature map 
\begin{equation}
\phi_{ARW}(G_{i})= (\text{Pr}[\text{a occurs in} G_{i}])_{a \in \mathcal{A}_{l}} 
\end{equation}
mapping into $R^{|\mathcal{A}_{l}|}$ ,where $\mathcal{A}_{l}$ denotes the finite set of possible anonymous walks of fixed length $l$. 
We can define the Anonymous walk kernel $K_{AW}$ by:
\begin{equation}
K_{AW}(G_i,G_j) = \langle \phi_{ARW}(G_{i}),\phi_{ARW}(G_{j}) \rangle
\end{equation}
$K_{AW}$ is a valid positive semi‑definite graph kernel since $\phi_{ARW}(G_{i})$  maps graphs into a Euclidean feature space and $K_{AW}$  is by construction symmetric and positive semi‑definite. In GKNNs, each layer performs a convolution-like operation comparing the input graph to a set of structural masks (subgraphs or patterns), via a graph kernel function:
\begin{equation}
h(G_{i}) = [K(G_{i},M_{1}), K(G_{i},M_{2}), \ldots, K(G_{i},M_{m})] 
\end{equation}
for trainable masks $M_{i}$ and any valid graph kernel can be used.
Taking $K=K_{AW}$ as the kernel in the GKNN layer. Then each component becomes:
\begin{equation}
K_{AW}(G,M_{j}) = \langle \phi_{ARW}(G),\phi_{ARW}(M_{j}) \rangle
\end{equation}
This yields exactly the similarity between $G$ and each mask in the ARW feature space. Feeding this similarity vector into a linear network (or stacking layers) is formally equivalent to constructing a neural architecture whose core is built on ARW distributions. Note that for any fixed $l$ for the length of RW we can fix a specific dimension for each mask. Hence, the Anonymous Walk Kernel is fully equivalent to the kernel used in GKNN when that kernel is chosen to be AWK, situating AWK as a special instance of the GKNN’s pluggable-kernel mechanism.
\end{proof}
It is remarkable to observe that Theorem~\ref{th:equivalence} shows that ARW modeling is a special case of GKNN modeling to obtain the structural node embedding for any downstream task such as graph classification or graph regression. 
\begin{theorem}
Let $G = (V, E, X)$ be a graph with node set $V$, edge set $E$, and node feature matrix $X \in \mathbb{R}^{|V| \times d}$. A Graph Attention Network (GAT) with global pooling is used to generate a graph-level representation $h_G$ for classification. If the GAT does not incorporate structural graph information (ARW embedding), then there exist non-isomorphic graphs $G_1$ and $G_2$ such that:
\begin{equation}
G_1 \not\simeq G_2 \quad \text{but} \quad h_{G_1} = h_{G_2}
\end{equation}
leading to misclassification and poor generalization.
\end{theorem}
\begin{proof}
Consider a GAT layer that updates node embeddings using self-attention. Let $h_i^{(l)}$ be the hidden representation of node $i$ at layer $l$. The update rule for GAT is given by:  
\begin{equation}
   h_i^{(l+1)} = \sigma \left( \sum_{j \in \mathcal{N}(i)} \alpha_{ij}^{(l)} W h_j^{(l)} \right)
\end{equation}
 where $W$ is a trainable weight matrix, $\sigma$ is a nonlinearity, and $\alpha_{ij}^{(l)}$ is the learned attention coefficient:
\begin{equation}
\alpha_{ij}^{(l)} = \frac{\exp\left(\text{LeakyReLU} (a^\top [W h_i^{(l)} \| W h_j^{(l)}])\right)}{\sum_{k \in \mathcal{N}(i)} \exp\left(\text{LeakyReLU} (a^\top [W h_i^{(l)} \| W h_k^{(l)}])\right)}
\end{equation}
 The attention mechanism allows nodes to weigh their neighbors differently but does not inherently incorporate global graph structure unless explicitly encoded. After $L$ layers of GAT, a global pooling function $P$ aggregates node embeddings into a single graph-level representation:
\begin{equation}
 h_G = P(\{ h_i^{(L)} \mid i \in V \})
\end{equation}
where $P$ is typically a sum, mean, or max function. Since $P$ is permutation-invariant, it treats graphs with the same set of node embeddings as identical. Now we aim to remove the structural ambiguity without structural information. Consider two non-isomorphic graphs $G_1$ and $G_2$ with the same node features but different structures. Since GAT message passing is purely feature-driven without explicit structural encoding, it follows that:
\begin{equation}
h_i^{(L)} (G_1) = h_i^{(L)} (G_2) \quad \forall i \in V
\end{equation}
leading to the same global representation:
\begin{equation}
h_{G_1} = P(\{ h_i^{(L)} (G_1) \}) = P(\{ h_i^{(L)} (G_2) \}) = h_{G_2}
\end{equation}
Since $G_1 \not\simeq G_2$ but $h_{G_1} = h_{G_2}$, the classifier cannot distinguish them, causing misclassification and poor generalization. To ensure that $G_1$ and $G_2$ are mapped to distinct embeddings, structural encodings (ARW embedding) must be involved in node representations:
 \begin{equation}\label{eq-concat}
 X' = [X \| ARW]
 \end{equation}
Incorporating $ARW$ alters the attention coefficients $\alpha_{ij}$ and the final embeddings $h_G$, ensuring that $h_{G_1} \neq h_{G_2}$, which improves generalization. 
\end{proof}

\begin{table}[h]
\centering
\caption{Hyperparameter settings used in all experiments.}
\label{tab:hyperparameters}
\small
\renewcommand{\arraystretch}{1.2}
\scalebox{1.0}{
\begin{tabular}{c c}
\hline
\textbf{Hyperparameter}         & \textbf{Value}       \\
\hline
Batch size                     & 32                  \\
Number of pooling layers       & 14                  \\
Number of attention heads ($h$)& 5                   \\
Epochs                         & 100                 \\
Hidden dimension               & 64                  \\
Learning rate                  & $1 \times 10^{-3}$  \\
Optimizer                      & Adam                \\
\hline
\end{tabular}
}
\end{table}

\section{Step-by-step: How Anonymous Walks Are Mapped to Embeddings}\label{ap:simple}
Let \( w_i \) be a random walk starting from a node in graph \( G \). 
The embedding \( \Phi(w_i) \in \mathbb{R}^d \) is obtained through the following steps:
\begin{enumerate}
    \item \textbf{Walk Generation} \par
    Sample a random walk from the graph, e.g., \( w = [v_1, v_2, v_1] \).

    \item \textbf{Anonymous Walk Encoding} \par
    Replace each node in the walk with the index of its first occurrence to get an anonymous pattern.
    For example, \( [v_1, v_2, v_1] \mapsto [0, 1, 0] \).

    \item \textbf{Pattern Indexing} \par
    Assign each unique anonymous walk pattern a unique ID from the set of all possible patterns.
    \begin{equation}
    \begin{split}
        \texttt{pattern\_to\_index}([0,1,0]) &= 0, \\
        \texttt{pattern\_to\_index}([0,1,2]) &= 1, \\
         \ldots  &= \ldots \\
    \end{split}
    \end{equation}

    \item \textbf{Embedding Lookup} \par
    Define an embedding matrix \( \Phi \in \mathbb{R}^{|\mathcal{A}| \times d} \), where \( \mathcal{A} \) is the set of all anonymous walk patterns. The walk embedding is retrieved via:
    \begin{equation}
        \Phi(w_i) = \Phi[\texttt{pattern\_to\_index}(w_i)]
    \end{equation}
A simple code is shown in figure~\ref{code:ARW} for implementation.
    \item \textbf{Training with Skip-Gram} \par
    The embeddings \( \Phi(w) \) are trained using the Skip-gram objective to maximize co-occurrence similarity between center-context walk pairs within a sliding window.
\end{enumerate}

\begin{figure*}[t]
\centering
\begin{minipage}{0.95\textwidth}
\begin{verbatim}
import torch
from torch import nn

# Example: assume 1000 unique walk patterns, 128-dim embeddings
embedding_table = nn.Embedding(num_embeddings=1000, embedding_dim=128)

# Suppose walk pattern [0,1,0] has index 0
walk_index = 0
walk_embedding = embedding_table(torch.tensor(walk_index))  # shape: [128]
\end{verbatim}
\end{minipage}
\caption{PyTorch code for anonymous walk embedding lookup.}
\label{code:ARW}

\end{figure*}

\section{Examples from QM9 dataset}\label{ap:QM9dataset}
Figure~\ref{fig:4atom15nodes} and figure~\ref{fig:5atom15nodes} show some molecules in QM9 dataset that have 4 atom types and 5 atom types respectively. The learned filters should cover different scales of structures such as filters with 5,6,7, and 8 nodes that are shown in figures~\ref{fig:5nodefilters}, \ref{fig:6nodefilters},\ref{fig:7nodefilters}, and \ref{fig:8nodefilters} respectively.
\begin{figure}[H]
    \centering
    \fbox{\includegraphics[width=0.3\textwidth]{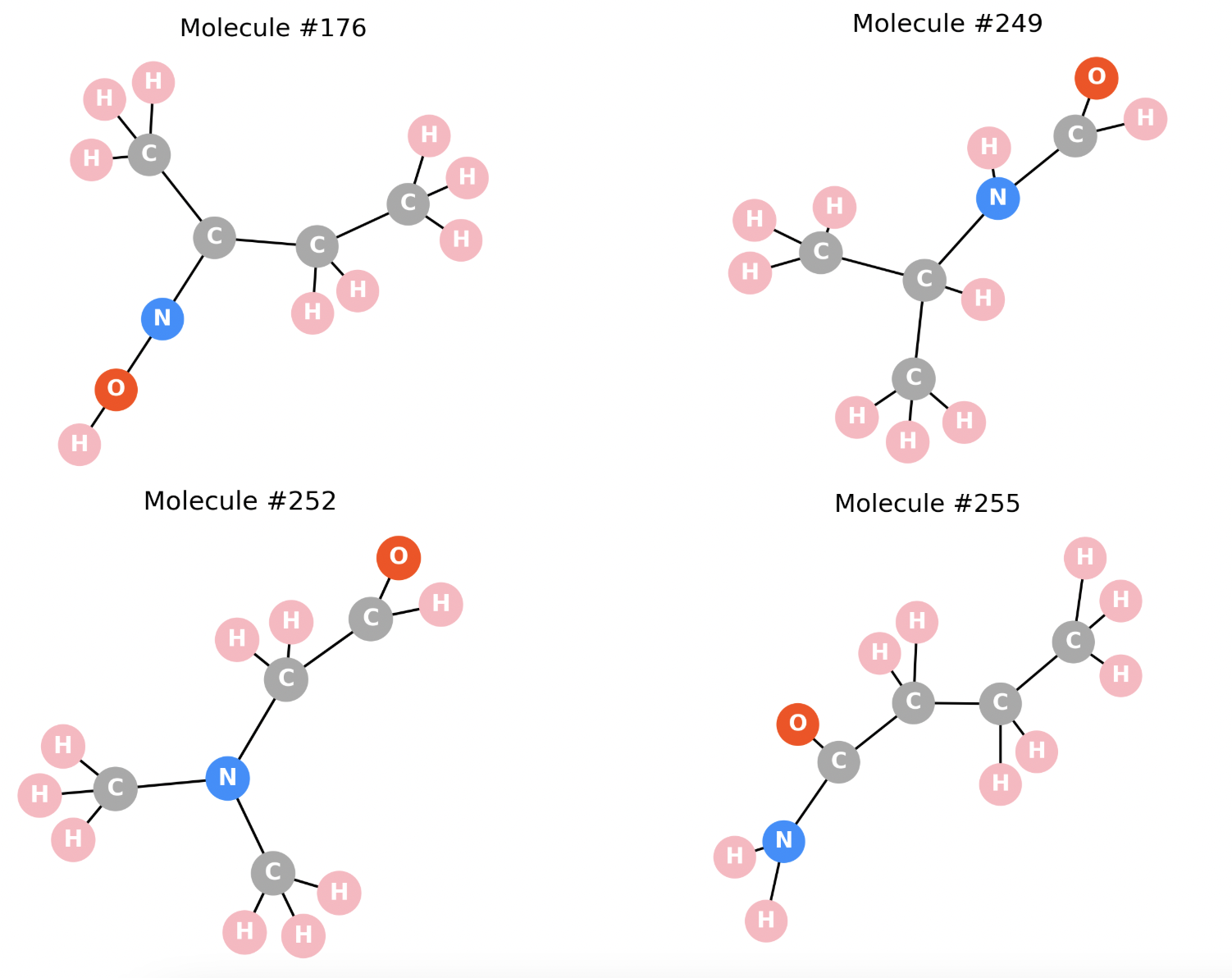}}
    \caption{Molecules in QM9 with 4 atom types and 15 nodes}
    \label{fig:4atom15nodes}
\end{figure}

\begin{figure}[H]
    \centering
    \fbox{\includegraphics[width=0.3\textwidth]{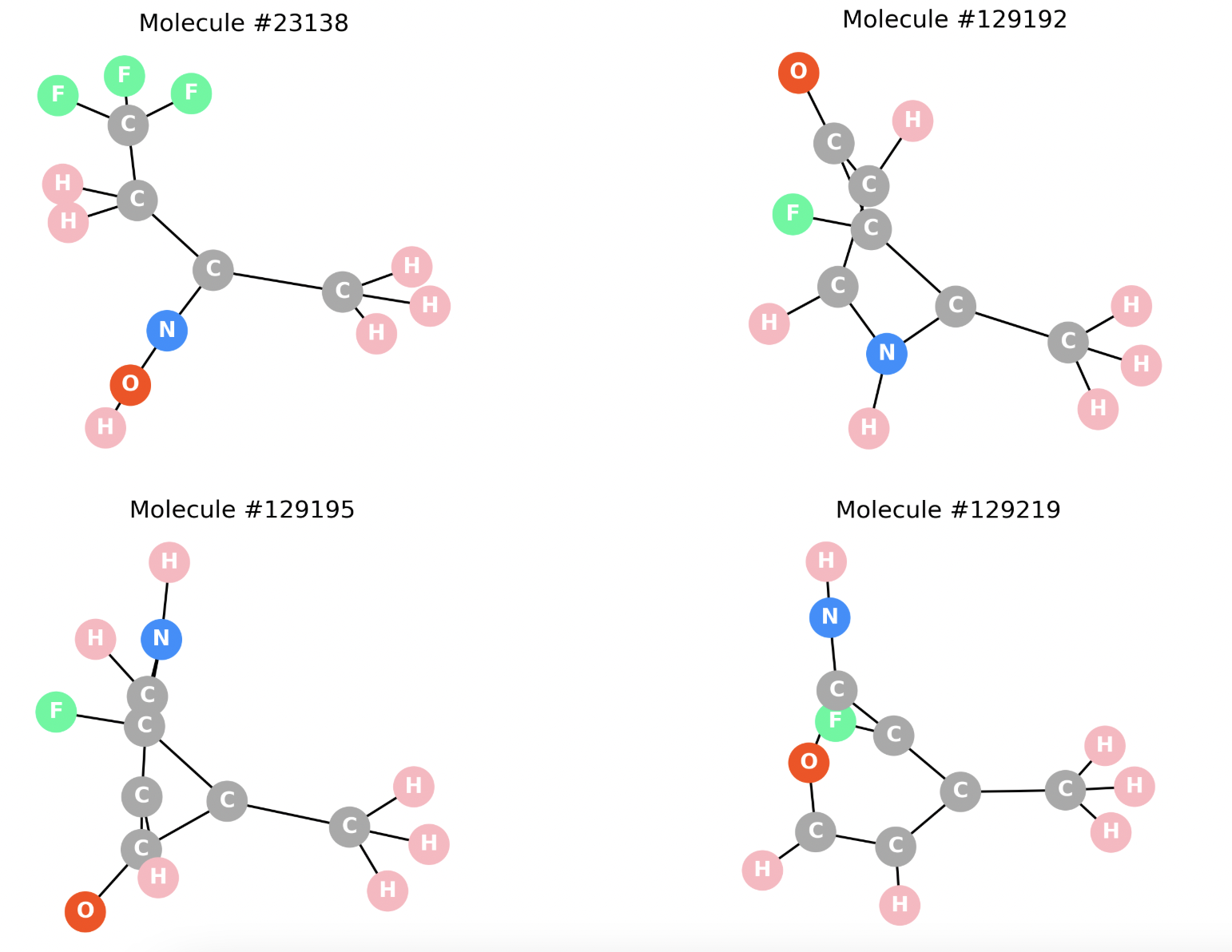}}
    \caption{Molecules in QM9 with 5 atom types and 15 nodes}
    \label{fig:5atom15nodes}
\end{figure}

\vspace{1em}

\section{Learning Filters}
Figures~\ref{fig:5nodefilters},\ref{fig:6nodefilters}, \ref{fig:7nodefilters},\ref{fig:8nodefilters} show filters with 5,6,7 and 8 nodes respectively that are learned by our algorithm. Figure~\ref{fig:zincfilters} shows the learned filters for ZINC dataset.
\begin{figure}[H]
    \centering
    \fbox{\includegraphics[width=0.3\textwidth]{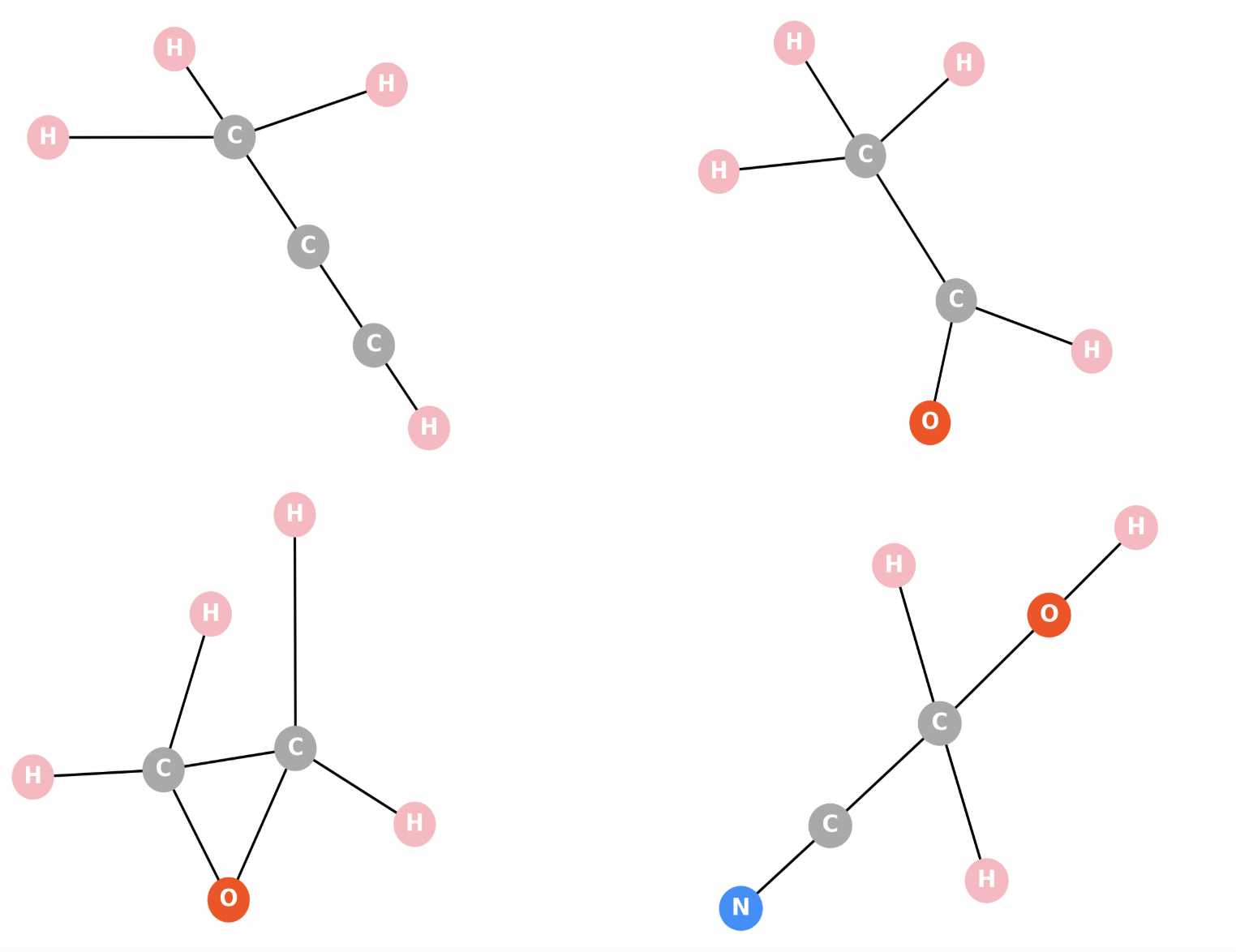}}
    \caption{7 node learned filters on QM9 dataset}
    \label{fig:7nodefilters}
\end{figure}

\vspace{1em}

\begin{figure}[H]
    \centering
    \fbox{\includegraphics[width=0.3\textwidth]{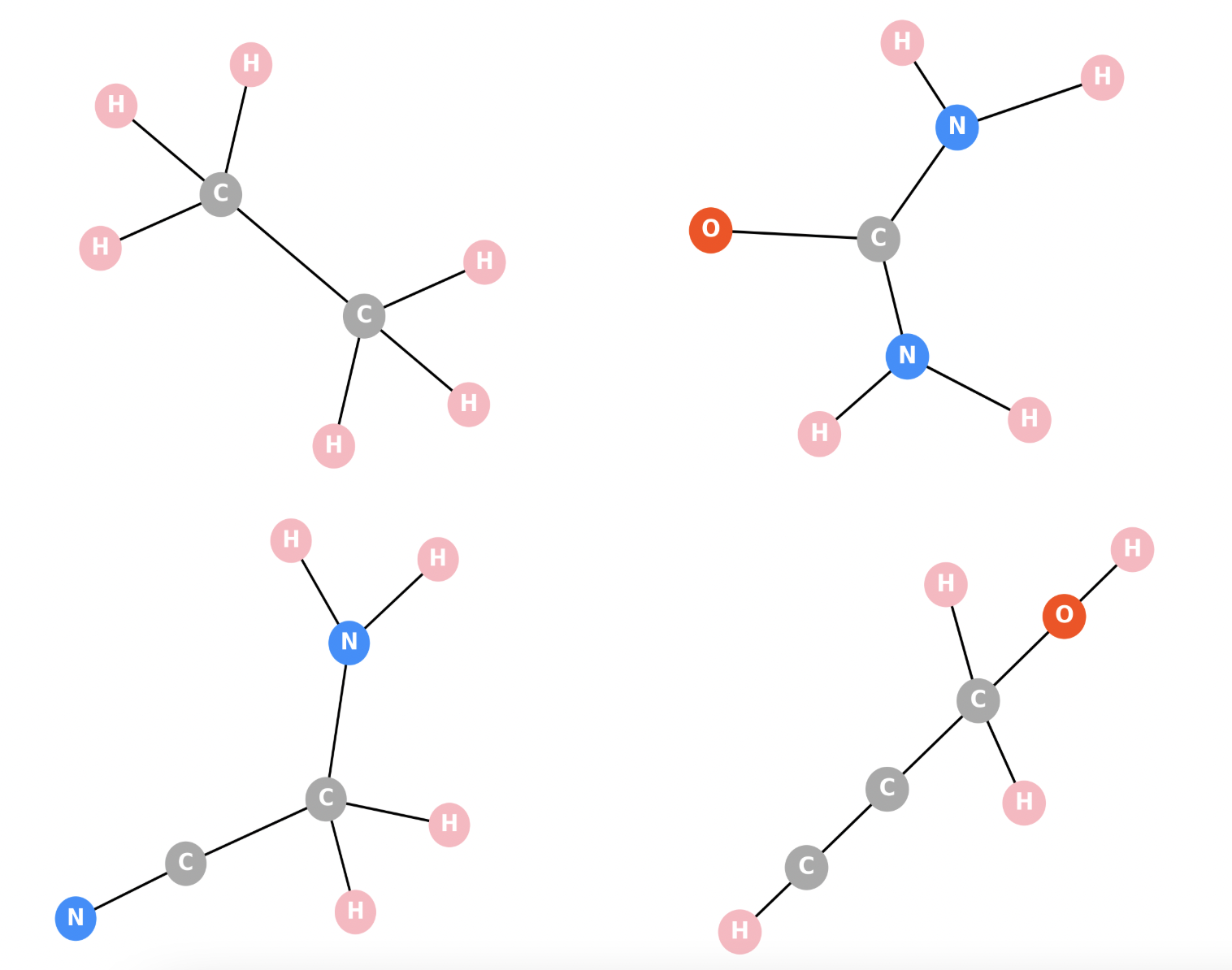}}
    \caption{8 node learned filters on QM9 dataset}
    \label{fig:8nodefilters}
\end{figure}

\begin{figure}
    \centering
    \fbox{\includegraphics[width=0.3\textwidth]{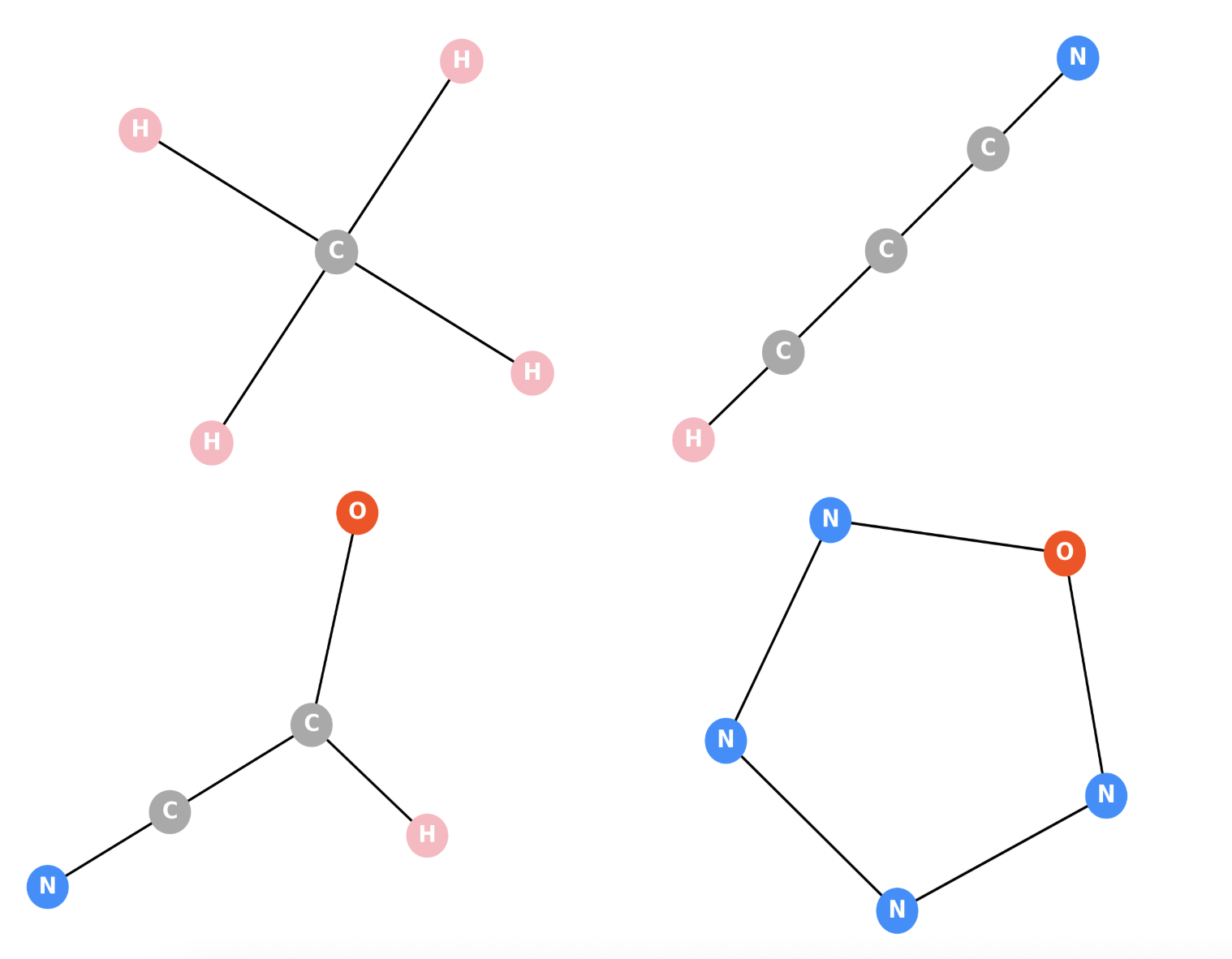}}
    \caption{5 node learned filters on QM9 dataset}
    \label{fig:5nodefilters}
\end{figure}

\vspace{1em}

\begin{figure}[H]
    \centering
    \fbox{\includegraphics[width=0.3\textwidth]{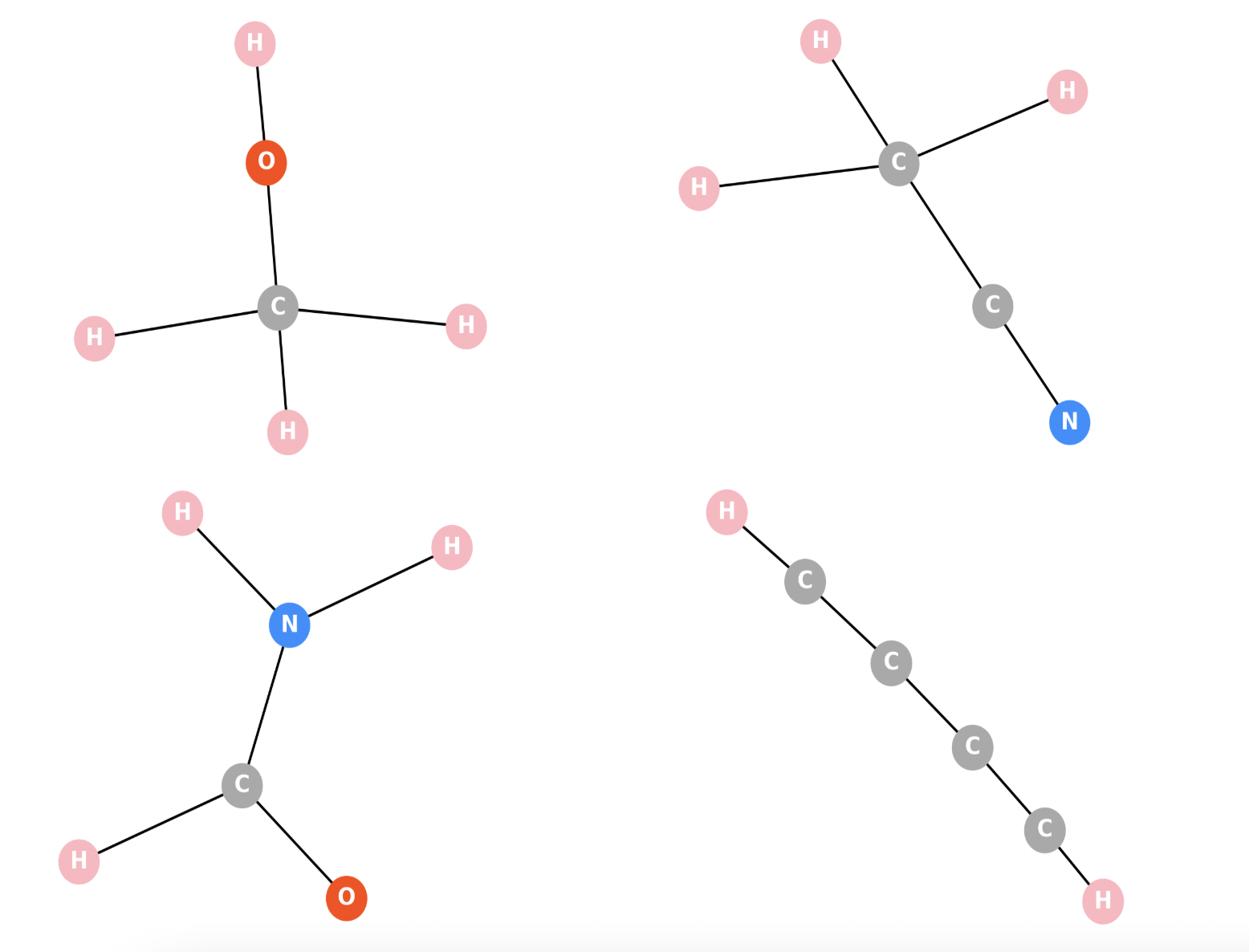}}
    \caption{6 node learned filters on QM9 dataset}
    \label{fig:6nodefilters}
\end{figure}

\vspace{1em}

\begin{figure}[H]
    \centering
    \fbox{\includegraphics[width=0.3\textwidth]{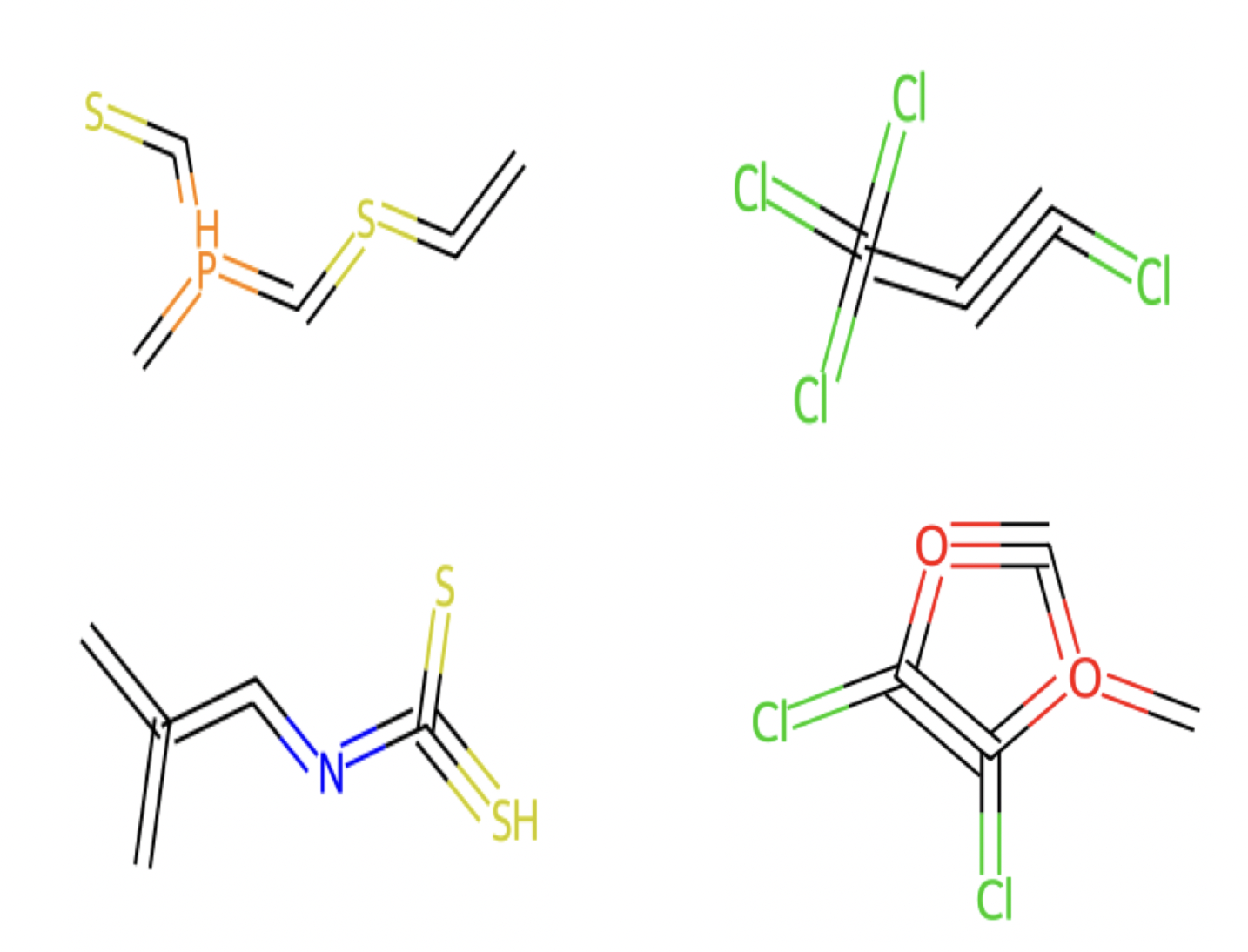}}
    \caption{some learned filters for ZINC dataset}
    \label{fig:zincfilters}
\end{figure}

\section{Ablation Study}
It can be observed in Table~\ref{tab:ablation_walklength_combined} that \texttt{walk\_length} plays an important role on performance of graph classification on Protein and MUTAG. Different \texttt{walk\_length} reveals different patterns that are responsible for graph rationales.
\begin{table*}[h]
\centering
\caption{Ablation study on the effect of \texttt{walk\_length} on the Protein and MUTAG datasets (mean ± std, \%). Results are averaged over 10 runs. \textbf{Bold} indicates the best, \underline{underline} indicates the second best.}
\label{tab:ablation_walklength_combined}
\small
\renewcommand{\arraystretch}{1.2}
\scalebox{0.95}{
\begin{tabular}{c c c c c}
\hline
\textbf{Dataset} & \textbf{Model} & \texttt{walk\_length}=10 & \texttt{walk\_length}=20 & \texttt{walk\_length}=40 \\
\hline
Protein & GIN-gp & 71.61 ± 1.06 & 71.43 ± 1.26 & 71.32 ± 1.46 \\
Protein & GCN-gp & \underline{72.61 ± 3.36} & \underline{72.43 ± 1.36} & \underline{72.52 ± 3.36} \\
Protein & GAT-gp & \textbf{72.83 ± 1.59} & \textbf{72.71 ± 1.14} & \textbf{72.61 ± 2.12} \\
\hline
MUTAG & GIN-gp & 82.11 ± 1.04 & 82.43 ± 1.27 & 82.12 ± 1.45 \\
MUTAG & GCN-gp & \underline{82.28 ± 3.25} & \underline{82.11 ± 1.24} & \underline{82.35 ± 1.70} \\
MUTAG & GAT-gp & \textbf{82.61 ± 1.49} & \textbf{82.31 ± 3.18} & \textbf{82.37 ± 2.12} \\
\hline
\end{tabular}
}
\end{table*}

\section{Graph Embedding Visualization}
Figure~\ref{fig:tsne-QM9-dipole-test} and figure~\ref{fig:tsne-QM9-isotropic-test} shows t-sne visualilization of the embedding vectors in supervised learning of dipole moment and isotropic polarizability property prediction respectively. Note that these figures are the graph embeddings at test time only. See the appendix~\ref{ap:QM9dataset} for training time.
\begin{figure}[H]
    \centering
    \fbox{\includegraphics[width=0.3\textwidth]{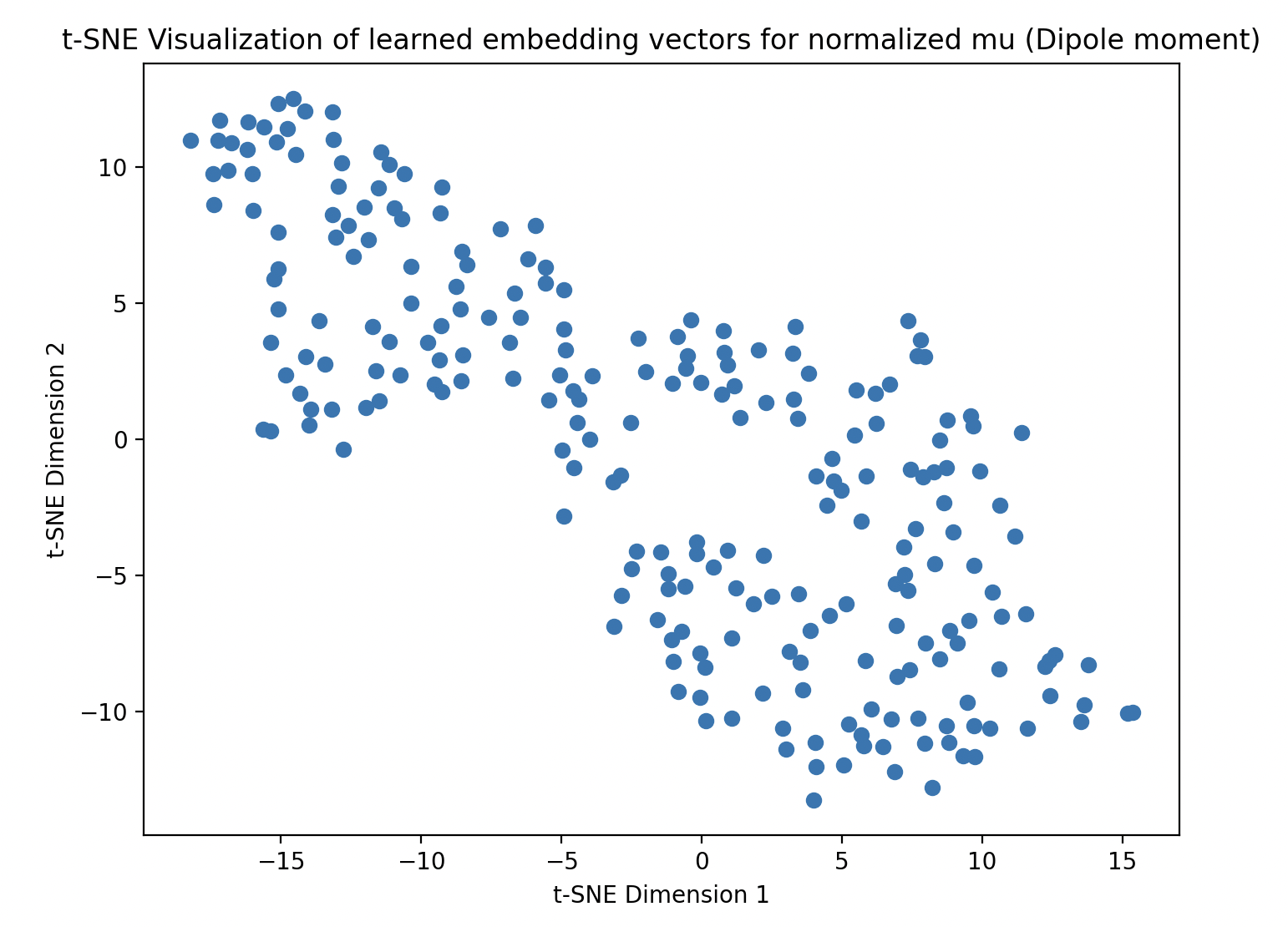}}
    \caption{2D visualization of graph embeddings based on dipole moment supervision on test data.}
    \label{fig:tsne-QM9-dipole-test}
\end{figure}

\begin{figure}[H]
    \centering
    \fbox{\includegraphics[width=0.3\textwidth]{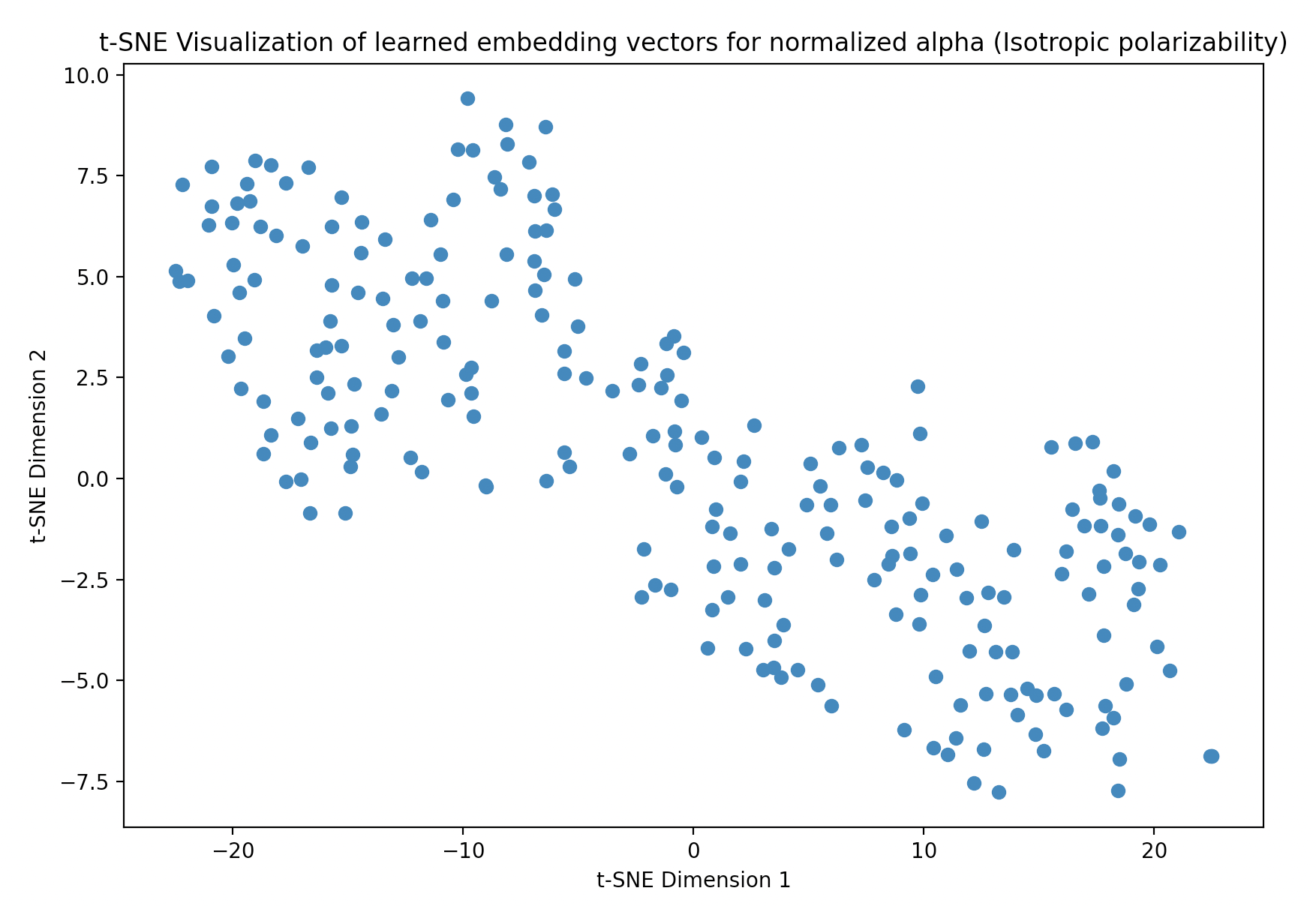}}
    \caption{2D visualization of graph embeddings based on Isotopic polarizability supervision on training data.}
    \label{fig:tsne-QM9-isotropic-test}
\end{figure}

Figure~\ref{fig:tsne-QM9-isotropic-training} shows the graph embedding of training data after mean pooling over node embeddings of individual graphs. 
\begin{figure}[H]
    \centering
    \fbox{\includegraphics[width=0.3\textwidth]{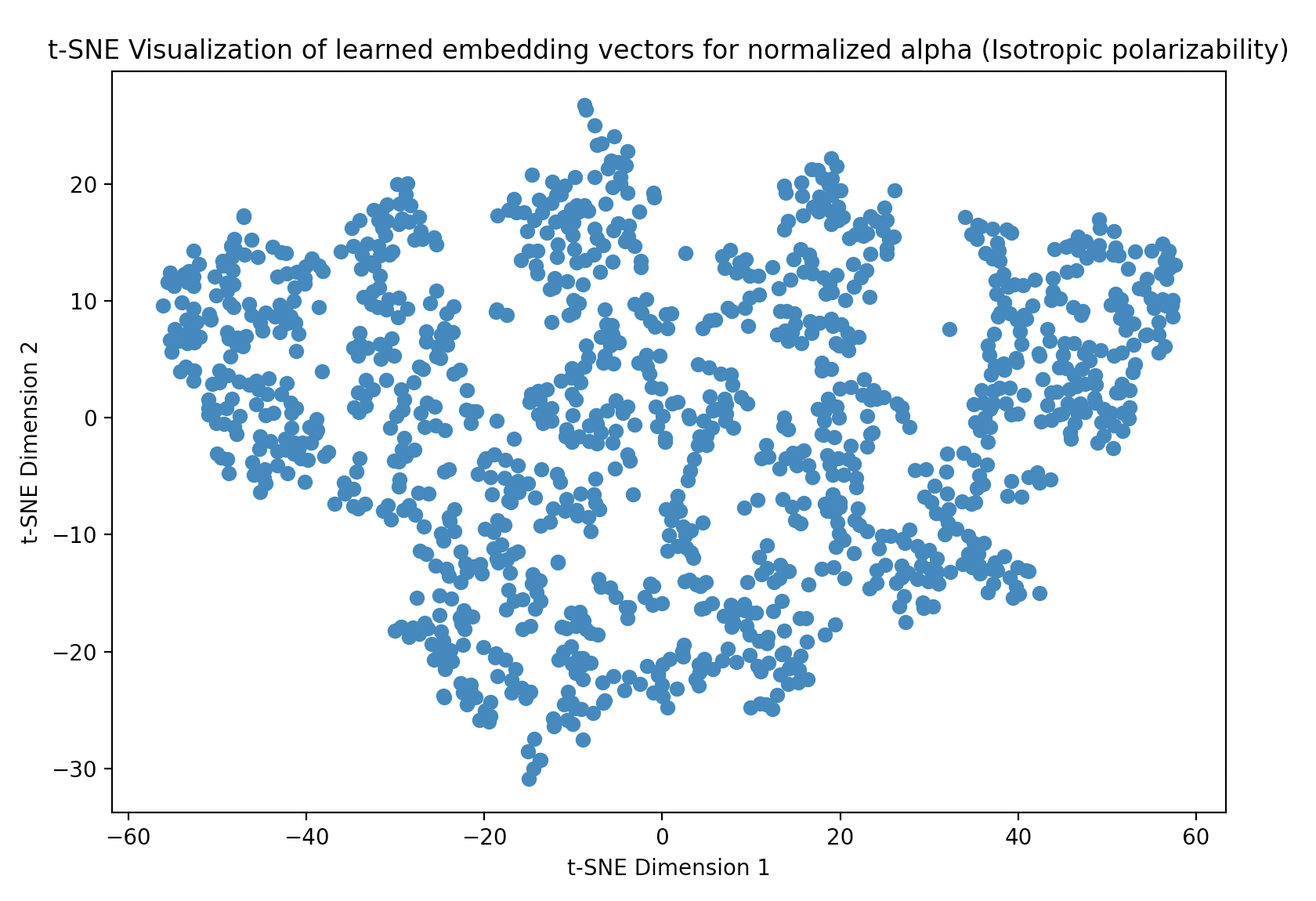}}
    \caption{2D visualization of graph embeddings based on isotropic polarizability supervision.}
    \label{fig:tsne-QM9-isotropic-training}
\end{figure}

\section{examples from ZINC dataset}\label{ap:ZINCdataset}
Figure~\ref{fig: zinc-data} show some molecules from ZINC dataset.  
\begin{figure}[H]
    \centering
    \fbox{\includegraphics[width=0.3\textwidth]{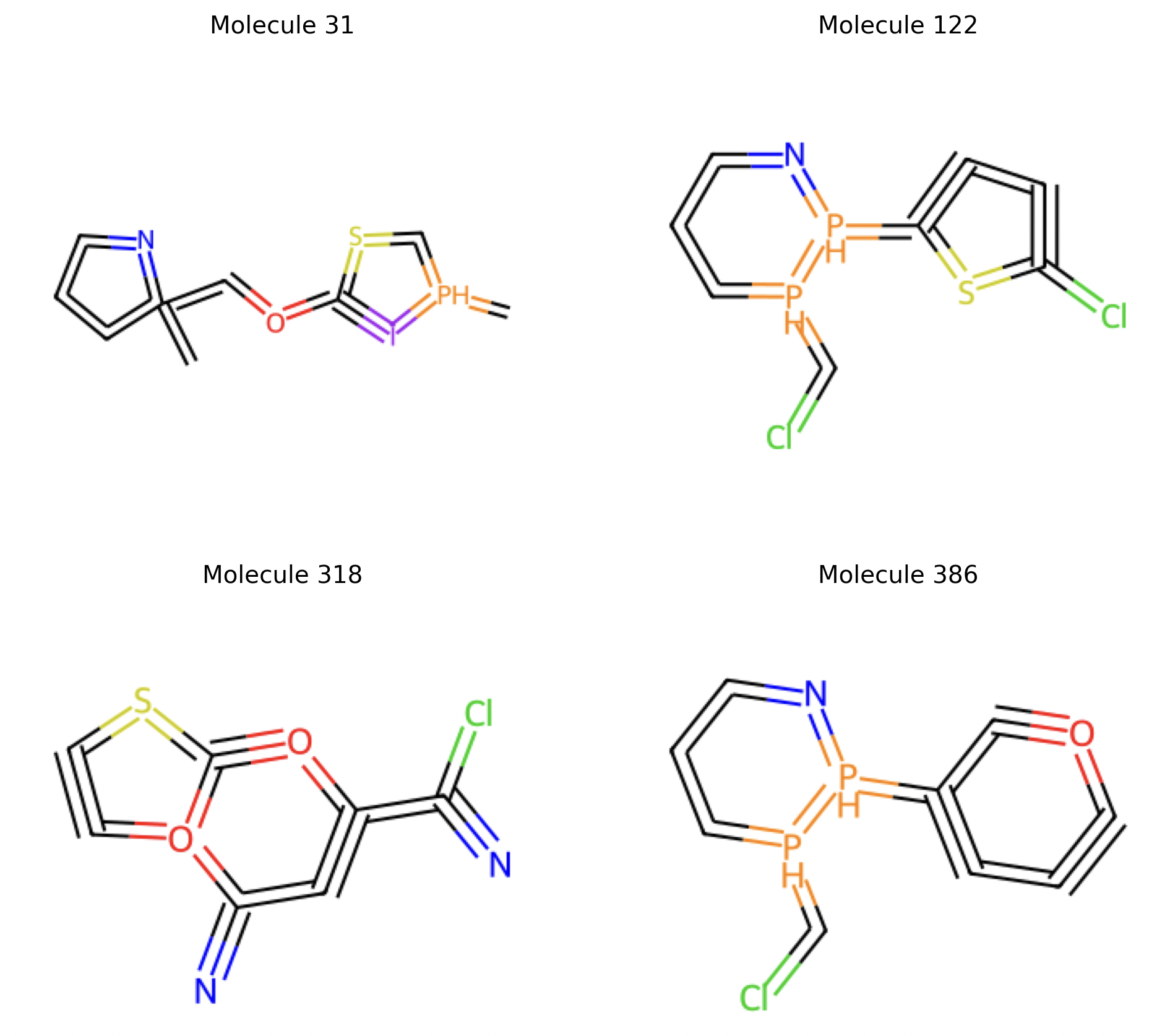}}
    \caption{examples from ZINC dataset}
    \label{fig: zinc-data}
\end{figure}

\section{ZINC Features}
In the ZINC molecular graph dataset, node features correspond to atom-level properties. Each node is encoded by a small set of discrete chemical attributes, typically represented as integers or one-hot embeddings before being fed into a GNN.
\begin{table}[h]
\centering
\caption{Node features (atom types) in the ZINC dataset. Each node is represented as a one-hot vector over 28 possible atom types.}
\label{table:zincfeatures}
\begin{tabular}{ll}
\hline
Index & Atom Type \\ \hline
1  & Br  \\
2  & C   \\
3  & Cl  \\
4  & F   \\
5  & I   \\
6  & N   \\
7  & O   \\
8  & P   \\
9  & S   \\
10 & Se  \\
11 & Si  \\
12 & B   \\
13 & H   \\
14 & Li  \\
15 & Na  \\
16 & K   \\
17 & Ca  \\
18 & Al  \\
19 & Fe  \\
20 & As  \\
21 & Zn  \\
22 & Cu  \\
23 & Ni  \\
24 & Mn  \\
25 & Co  \\
26 & Cr  \\
27 & Mg  \\
28 & Pd  \\
29 & Sn  \\ \hline
\end{tabular}
\label{tab:zinc_node_features}
\end{table}

\end{document}